\documentclass[sts]{imsart}

\RequirePackage{amsthm,amsmath,amsfonts,amssymb}
\RequirePackage[numbers,sort&compress]{natbib}
\usepackage{booktabs}
\usepackage{makecell}
\RequirePackage[colorlinks,citecolor=blue,urlcolor=blue]{hyperref}
\RequirePackage{graphicx}

\usepackage{enumitem}

\theoremstyle{plain}

\newtheorem{theorem}{Theorem}[section]
\newtheorem{lemma}[theorem]{Lemma}

\newtheorem{assumption}{Assumption}
\theoremstyle{remark}

\newtheorem{remark}{Remark}
\newtheorem{proposition}{Proposition}

\def\E{\mathbb{E}}
\def\P{\mathbb{P}}

\def\Var{\mathrm{Var}}

\def\R{\mathbb{R}}
\def\cA{\mathcal{A}}

\usepackage{amsmath}
\newcommand\norm[1]{\left\lVert#1\right\rVert}

\newcommand{\argmin}{\mathop{\mathrm{argmin}}}

\endlocaldefs

\begin{document}

\begin{frontmatter}
\title{On the Statistical Complexity for Offline and Low-Adaptive Reinforcement Learning with Structures}
\runtitle{Statistical Complexity for Offline and Low-Adaptive RL }

\begin{aug}
\author[A]{\fnms{Ming}~\snm{Yin}\ead[label=e1]{my0049@princeton.edu}},
\author[B]{\fnms{Mengdi}~\snm{Wang}\ead[label=e2]{mengdiw@princeton.edu }}
\and
\author[C]{\fnms{Yu-Xiang}~\snm{Wang}\ead[label=e3]{yuxiangw@ucsd.edu}}


\address[A]{Ming Yin is a Postdoc at the Department of Electrical and Computer Engineering at Princeton
University\printead[presep={\ }]{e1}.}

\address[B]{Mengdi Wang is an Associate Professor at the Department of Electrical and Computer Engineering at Princeton University\printead[presep={\ }]{e2}.}

\address[C]{Yu-Xiang Wang is an Associate Professor at the Halıcıoğlu Data Science Institute at UC San Diego\printead[presep={\ }]{e3}.}

\end{aug}

\begin{abstract}
This article reviews the recent advances on the statistical foundation of reinforcement learning (RL) in the offline and low-adaptive settings.  We will start by arguing why offline RL is the appropriate model for almost any real-life ML problems, even if they have nothing to do with the recent AI breakthroughs that use RL. Then we will zoom into two fundamental problems of offline RL:  offline policy evaluation (OPE) and offline policy learning (OPL). It may be surprising to people that tight bounds for these problems were not known even for tabular and linear cases until recently. We delineate the differences between worst-case minimax bounds and instance-dependent bounds. We also cover key algorithmic ideas and proof techniques behind near-optimal instance-dependent methods in OPE and OPL. Finally, we discuss the limitations of offline RL and review a burgeoning problem of \emph{low-adaptive exploration} which addresses these limitations by providing a sweet middle ground between offline and online RL. 

\end{abstract}


\end{frontmatter}

\section{Introduction}

Reinforcement learning (RL) has gained remarkable popularity lately. Most people would attribute the surge to its usage in AI milestones such as AlphaGo \cite{mnih2015human,silver2016mastering,silver2017mastering,fawzi2022discovering,mankowitz2023faster} and in instruction-tuning large language models \cite{christiano2017deep,stiennon2020learning,ouyang2022training,bai2022training}. We, however, argue that it is caused by a more fundamental paradigm shift that places RL in the front and center of nearly every Machine Learning (ML) application in practice. Why? Training an accurate classifier is most likely not the end goal of an ML task. Instead, the predictions of the trained ML model is often used as interventions hence changing the distribution of future data. Real-world applications are usually sequential decision-making problems, and trained ML models need to be combined with RL methods to perform high-quality decision-making. We provide three examples. 

\textit{AI Diagnosis/Screening.} In medical diagnosis, ML models are frequently used to predict the likelihood of a patient having a certain disease based on their symptoms and medical history. However, these predictions are not the final outcome; they often guide subsequent medical interventions, such as recommending further tests or treatments. These interventions, in turn, influence future patient states, creating a feedback loop that affects the data distribution. RL methods are essential in this context to optimize the sequence of decisions—like treatment plans—over time, improving patient's outcomes. For instance, \cite{nemati2016optimal} used RL to develop a model that assists in the management of ICU by recommending treatment strategies that adapt to the evolving condition of the patient.

\textit{Recommendation Systems.} Traditional recommendation systems rely on ML models to predict user preferences based on historical data. However, when these recommendations are presented to users, they influence user behavior and preferences, which alters future data. This dynamic environment is well-suited to RL, where the goal is to maximize long-term user engagement by continuously adapting recommendations based on real-time feedback. For example, \cite{zhao2018deep} applied RL to optimize a recommendation system for news articles, showing that it could significantly improve user click-through rates by considering the long-term effects of recommendations.

\textit{Video Streaming over Wireless Networks.} In video streaming applications, ML models are used to predict network conditions and select appropriate streaming bitrates. These predictions directly influence the quality of the streaming experience and the subsequent network load, posing a challenging sequential decision-making problem. RL can be applied to adaptively adjust bitrates to optimize the trade-off between video quality and buffering. For instance, \cite{mao2017neural} introduced a system called \emph{Pensieve}, which uses RL to optimize video streaming quality over wireless networks by learning from past streaming experiences and network conditions.

These examples not only demonstrate the fundamental applicability of RL across diverse domains but also bring to light the significant challenges it faces. 

Notably, most real-life RL problems are \emph{offline RL} problems. Unlike Chess or Go with unlimited access to simulators, it is often unsafe, illegal, or costly to conduct experiments in the task environment. Instead, we need to work with an offline dataset collected from the environment, which poses fundamental problems on \emph{what can be learned offline} and \emph{how (statistically) efficiently one can learn from the offline dataset}. Three critical aspects of the offline RL problems are:
\begin{itemize}
\item \textbf{Long horizon problem.} The long decision horizon in RL poses unique challenge for finding the optimal strategy. In particular, the undesired actions chosen at earlier phases will have long-lasting impact for the future, making the strategy suboptimal. Small deviations from the optimal policy early on can propagate and amplify over time, further complicating the learning process. 

\item \textbf{Distribution Shift and Coverage.} Distribution shift is a fundamental challenge in reinforcement learning that occurs when the distribution of data the agent encounters during training differs from the distribution of optimal policies. When the overlap (measured by certain distribution distance metric) between the two distributions is small, it would be hard to find optimal actions due to the insufficient data coverage, especially when the offline dataset is collected using a suboptimal policy. 
\item \textbf{Function Approximation and Generalization.} The state and action space of RL problems are often so large that a finite dataset cannot cover. In such cases, RL requires generalization across states through a certain feature representation of the states and a parametric approximation of the value functions. Learning such function approximations offline is more challenging.
\end{itemize} 

This article aims to review recent advances of the statistical foundations for offline RL, covering both problems in \emph{offline policy evaluation} and \emph{offline policy learning}. Specifically, we review what the fundamental learning hardness/statistical limits for offline RL under different MDP (Markov Decision Processes) or function approximation structures are. By examining the statistical results, we reveal how factors such as distribution shift and horizon length affect the learning hardness of the problems. We also introduce the related algorithms and highlight the theoretical techniques to achieve these results. 

\textbf{Paper organization.} We first introduce the mathematical notations and set up the problems of interest in Section~\ref{sec:setup}. Then the remaining sections 
describe results in offline policy evaluation, offline policy learning and low-adaptive exploration under various assumptions (see Table~\ref{tab:organization}). 

\begin{table*}[t]
    \centering
    \begin{tabular}{c|ccc}
        Problem setup & Offline Evaluation &  Offline Learning & Low-Adaptive Exploration\\
         \hline
       Tabular MDP  & Section~\ref{sec:ope_tabular} &Section~\ref{sec:OPL}&Section~\ref{sec:batched_tabular}\\
       Linear Approx. &Section~\ref{sec:ope_linear} & Section~\ref{sec:OPL_linear} &Section~\ref{sec:batched_linear}\\
       Parametric Approx. &Section~\ref{sec:ope_parametric} & Section~\ref{sec:OPL_parametric} &Section~\ref{sec:batched_general}\\
    \end{tabular}
    \caption{Overview of the paper structure.}
    \label{tab:organization}
\end{table*}

\begin{figure}
    \centering
    \includegraphics[width=1.0\linewidth]{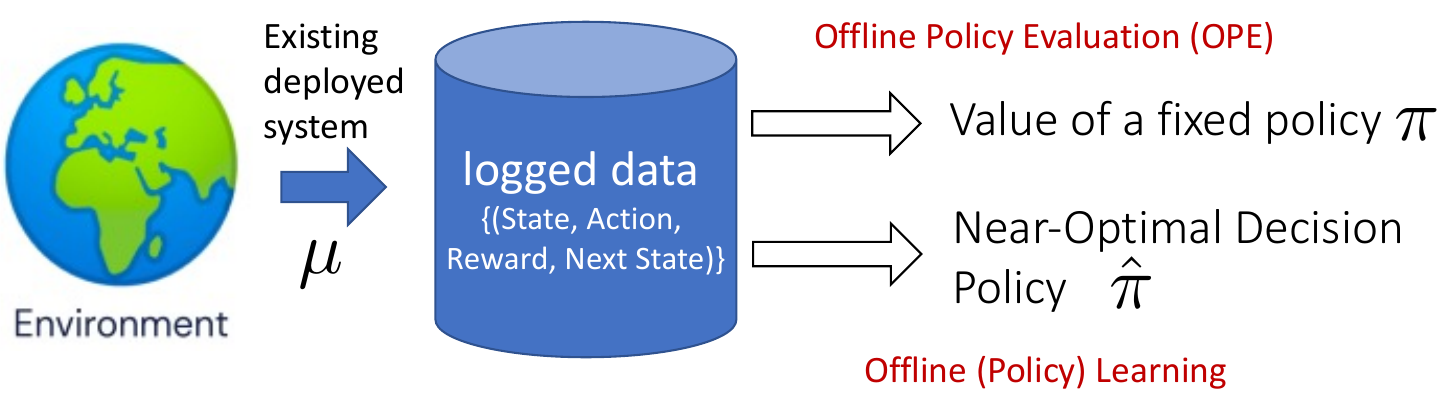}
    \caption{Illustration of the offline reinforcement learning problem.}
    \label{fig:illus_offlinerl}
\end{figure}
\textbf{Disclaimer.} The literature of offline RL is gigantic. It is not our intention to provide thorough coverage. Instead, the topics and results covered in this paper focus on a niche that the coauthors studied in the past few years.  Since our goal is pedagogical, we do not make any claims about novelty and precedence of scientific discovery. Please refer to the bibliography and the references therein for a more detailed discussion.


\section{Notations and problem setup} \label{sec:setup}
We first provide the background for different problem settings 
that we consider in this article. 

\subsection{Episodic time-inhomogenuous RL} A finite-horizon \emph{Markov Decision Process} (MDP) is denoted by a tuple $\mathcal{M}=(\mathcal{S}, \mathcal{A}, P, r, H, d_1)$ \citep{sutton2018reinforcement}, where $\mathcal{S}$ is the state space and $\mathcal{A}$ is the action space. A time-inhomogenuous transition kernel $P_h:\mathcal{S}\times\mathcal{A}\times\mathcal{S} \mapsto [0, 1]$ maps each state action$(s_h,a_h)$ to a probability distribution $P_h(\cdot|s_h,a_h)$ and $P_h$ can be different across the time. Besides, $r : \mathcal{S} \times{A} \mapsto \mathbb{R}$ is the expected instantaneous reward function satisfying $0\leq r\leq R_{\max}$. $d_1$ is the initial state distribution. $H$ is the horizon. A policy $\pi=(\pi_1,\ldots,\pi_H)$ assigns each state $s_h \in \mathcal{S}$ a probability distribution over actions according to the map $s_h\mapsto \pi_h(\cdot|s_h)$ $\forall h\in[H]$.  An MDP together with a policy $\pi$ induce a random trajectory $ s_1, a_1, r_1, \ldots, s_H,a_H,r_H,s_{H+1}$ with $s_1 \sim d_1, a_h \sim \pi(\cdot|s_h), s_{h+1} \sim P_h (\cdot|s_h, a_h), \forall h \in [H]$ and $r_h$ is a random realization given the observed $s_h,a_h$.

\emph{Bellman (optimality) equations.} The value function $V^\pi_h(\cdot)\in \R^\mathcal{S}$ and Q-value function $Q^\pi_h(\cdot,\cdot)\in \R^{\mathcal{S}\times \mathcal{A}}$ for any policy $\pi$ is defined as, $\forall s,a\in\mathcal{S},\mathcal{A},h\in[H]$:
{\small
\begin{align*}
V^\pi_h(s)=\E_\pi[\sum_{t=h}^H r_{t}|s_h=s],\;\;Q^\pi_h(s,a)=\E_\pi[\sum_{t=h}^H  r_{t}|s_h,a_h=s,a].
\end{align*}
}The {Dynamic Programming principle} follows \cite{puterman1990markov,bellman1966dynamic,powell2007approximate} $\forall h\in[H]$:
{\small
\begin{equation}\label{eqn:bellman}
\begin{aligned}
Q^\pi_h(s,a)&=r_h+\E_{s'\sim P_h(\cdot|s,a)}[V^\pi_{h+1}(s')],\;\;V^\pi_h=\E_{a\sim\pi_h}[Q^\pi_h],\\ \;\;\;Q^*_h(s,a)&=r_h+\E_{s'\sim P_h(\cdot|s,a)}[V^*_{h+1}(s')],\; V^*_h=\max_a Q^*_h(\cdot,a).
\end{aligned}
\end{equation}
}The corresponding Bellman operators are defined as: 
{\small
\begin{equation}\label{eqn:operator}
\begin{aligned}
\mathcal{P}^\pi_h(f)(s,a)&=r_h+\E_{s'\sim P_h(\cdot|s,a),a'\sim \pi(\cdot|s')}[f(s',a')],\\ \mathcal{P}_h(f)(s,a)&=r_h+\E_{s'\sim P_h(\cdot|s,a)}[\max_{a'}f(s',a')].
\end{aligned}
\end{equation}
}We incorporate the standard marginal state-action occupancy $d^\pi_h(s,a)$ as:
{$
d^\pi_h(s,a):=\P[s_h=s,a_h=a|s_1\sim d_1,\pi].
$
}The performance per policy $\pi$ is defined as $$v^\pi:=\E_{d_1}\left[V^\pi_1\right]=\E_{\pi,d_1}\left[\sum_{t=1}^H  r_t\right].$$

\subsection{Structured MDP models}\label{subsec:str}

In this article, we examine three 
fundamental yet representative MDP models (or related function approximation classes) that are well-structured. Despite their simplicity, as we will discuss in later sections, their statistical limits have not been well-understood until recently.

\textbf{Tabular MDPs.} Tabular MDP is arguably the most simple setting in RL. It is a Markov Decision Process with finite states $|\mathcal{S}|<\infty$ and finite actions $|\mathcal{A}|<\infty$. The most common tabular MDPs, such as Gridworlds, often have small state and action spaces. When the number of states and actions are large, they are generally not treated as discrete and are instead addressed using function approximators. 

\textbf{Linear MDPs.} An episodic MDP $(\mathcal{S},\mathcal{A},P,r,H,d_1)$ is called a linear MDP with a known (unsigned) feature map $\phi:\mathcal{S}\times\mathcal{A}\rightarrow \mathbb{R}^d$ if there exist $d$ unknown (unsigned) measures $\nu_h=(\nu_h^{(1)},\ldots,\nu_h^{(d)})$ over $\mathcal{S}$ and an unknown vector $\theta_h\in\R^d$ such that $\forall s',s\in\mathcal{S}, \;a\in\mathcal{A}, \;h\in[H]$,
	\[
	{P}_{h}\left(s^{\prime} \mid s, a\right)=\left\langle\phi(s, a), \nu_{h}\left(s^{\prime}\right)\right\rangle, \; r_{h}\left(s, a\right) =\left\langle\phi(x, a), \theta_{h}\right\rangle
	\]
	with $\int_\mathcal{S}\norm{\nu_h(s)}ds\leq\sqrt{d}$ and $\max(\norm{\phi(s,a)}_2,\norm{\theta_h}_2)\leq 1$ for all $h\in[H]$ and $\forall s,a\in\mathcal{S}\times\mathcal{A}$. 

When specify $d=|\mathcal{S}| \times|\mathcal{A}|$ and $\phi(x, a)=\mathbf{1}_{(x, a)}$ be the canonical basis in $\mathbb{R}^d$, linear MDPs recover tabular MDPs. Thus, linear MDPs strictly generalize tabular MDPs and allow continuous state-actions spaces.

\textbf{Linear Functions approximation.} 
By Bellman equation \eqref{eqn:bellman}, Linear MDP aimplies that the value function $Q_h^\pi$ for any policy $\pi$ is a linear function in the feature vector $\phi$, i.e., 
$$Q_h^\pi(\cdot,\cdot)\in \left\{ \langle \phi(\cdot,\cdot),\theta \rangle \;|\; \theta\in \R^d\right\}$$ for any $h\in[H]$ and $\pi$. 
It is sometimes sufficient to directly reason about these linear function approximations rather than relying on the stronger linear MDPs structures. There are various subtle differences in the various type of linear function approximation. For the purpose of this paper though, it suffices to just think about linear MDPs.

For more general MDPs, it is harder to impose tractable structures. Alternatively, we consider the following structured function class that is expressive enough to learn general MDPs.

\textbf{Parametric Differentiable Functions.} Let $\mathcal{S},\mathcal{A}$ be arbitrary state, action spaces and a feature map $\phi(\cdot,\cdot):\mathcal{S}\times\mathcal{A}\rightarrow \Psi\subset\mathbb{R}^m$. The parameter space $\Theta\in\mathbb{R}^d$. Both $\Theta$ and $\Psi$ are compact spaces. Then the parametric function class (for a model $f:\R^d\times \R^m\rightarrow \R$) is defined as
	\[
	\mathcal{F}:=\{f(\theta,\phi(\cdot,\cdot)):\mathcal{S}\times\mathcal{A}\rightarrow\mathbb{R},\theta\in\Theta\}
	\]
	that satisfies differentiability/smoothness condition: 1. for any $\phi\in \R^m$, $f(\theta,\phi)$ is third-time differentiable with respect to $\theta$; 2. $f,\partial_\theta f,\partial^2_{\theta,\theta} f,\partial^3_{\theta,\theta,\theta} f$ are jointly continuous for $(\theta,\phi)$. 
Clearly, $\mathcal{F}$ generalizes linear function class (via choosing $f(\theta,\phi)=\langle \theta,\phi\rangle$). 


\subsection{Offline RL Tasks} The offline RL begins with a static offline data {\small$\mathcal{D}=\left\{\left(s_{h}^{\tau}, a_{h}^{\tau}, r_{h}^{\tau}, s_{h+1}^{\tau}\right)\right\}_{\tau\in[n]}^{h\in[H]}$} rolled out from some behavior policy $\mu$. In particular, the offline nature requires we cannot change $\mu$ and in particular we do not assume the functional knowledge of $\mu$. There are two major tasks considered in offline RL.
\begin{itemize}
    \item \textbf{Offline Policy Evaluation (OPE).} For a policy of interest $\pi$, the agent needs to evaluate its performance $v^\pi$ using $\mathcal{D}$. In general, there is a distribution mismatch between $\pi$ and $\mu$. The goal is to construct an estimator $\widehat{v}^\pi$ such that $|v^\pi-\widehat{v}^\pi|<\epsilon$ or mean square error $\E_\mu[(v^\pi-\widehat{v}^\pi)^2]<\epsilon$).  
    
    \item \textbf{Offline Policy Learning (OPL).} This requires the agent to find a reward-maximizing policy $\pi^*:=\text{argmax}_\pi v^\pi$ given data {\small$\mathcal{D}$}. That is to say, given the batch data $\mathcal{D}$ and a targeted accuracy $\epsilon>0$, the offline RL seeks to find a policy $\pi_\text{alg}$ such that $v^*-v^{\pi_\text{alg}}\leq\epsilon$.

\end{itemize}

Both OPE and OPL are essential to a real-world offline RL system since the decision maker should first run the offline learning algorithm to find a near optimal policy and then use OPE methods to check if the obtained policy is good enough. For instance, in finance, OPL can be applied for learning a strategy, but traders still need to run OPE for backtesting before deployment. On the other hand, they are also standalone research questions, \emph{e.g.} doctors can be asked to evaluate a heuristic treatment plan that does not involve offline learning, which makes it a pure OPE problem.

\begin{remark}The source of historical data $\mathcal{D}$ could be multilateral, and there are papers (e.g. \cite{jin2021pessimism,ren2021nearly}) directly considers data distribution without specifying $\mu$. We incorporate a specific behavior policy $\mu$ to manifest the distribution mismatch between $\mu$ and $\pi$.
\end{remark}

\subsection{Assumptions in offline RL}

Due to the inherent distribution shift in offline RL, for both OPE and OPL, we revise different assumptions for different problem classes in Section~\ref{subsec:str}. These assumptions are standard protocols for deriving provably efficient results.

\textbf{Offline Policy Evaluation.} For tabular OPE, it requires marginal state ratios and policy ratios to be finite, as stated below.
\begin{assumption}[Tabular OPE \cite{xie2019towards,yin2020asymptotically}]\label{assum:tope}
    	Logging policy $\mu$ obeys that $d_m:=\min_{t,s}d^\mu_t(s)>0$.	Also, $\tau_s:=\max_{t,s}\frac{d^\pi_t(s)}{d^\mu_t(s)}<+\infty$ and $\tau_a:=\max_{t,s,a}\frac{\pi(a|s)}{\mu(a|s)}<+\infty$.
\end{assumption}
Having bounded weights is necessary for discrete state and actions, as otherwise the unbounded importance ratio would cause the estimation error
become intractable.
\begin{assumption}[Linear OPE \cite{duan2020minimax,hao2021bootstrapping}]\label{assume:lope}
    Let the population feature covariance $\Sigma_h:=\mathbb{E}_{\mu, h}\left[\phi(s, a) \phi(s, a)^{\top}\right]$. Then we assume $\min_h\lambda_{\text{min}}(\Sigma_h)>0$ with $\lambda_{\text{min}}$ being the minimal eigenvalue.
\end{assumption}
This assumption ensures the behavior policy $\mu$ has good coverage over the state-action spaces. For instance, when $\phi(x, a)=\mathbf{1}_{(x, a)}$, the assumption above reduces to $\min_{s,a}d_h^\mu(s,a)>0$.

\begin{assumption}[Parametric OPE \cite{zhang2022off}]\label{assume:pope}
Policy completeness: assume reward $r\in\mathcal{F}$ and for any $f\in\mathcal{F}$, we have $\mathcal{P}^{\pi} f\in\mathcal{F}$. Policy realizability: assume $Q^\pi(\cdot,\cdot)=f(\phi(\cdot,\cdot),\theta^\pi)$ for some $\theta^\pi\in\R^d$. Lastly, let the population feature covariance {\small$$\Sigma_h:=\mathbb{E}_{\mu, h}\left[\nabla f(\phi(s, a),\theta^\pi)\nabla f( \phi(s, a),\theta^\pi)^{\top}\right].$$} Then we assume $\min_h\lambda_{\text{min}}(\Sigma_h)>0$.
\end{assumption}

Policy completeness and policy realizability ensure the policy class is rich enough to capture $Q^\pi$. Besides, the assumption on the population feature covariance generalizes the Linear OPE case.

\textbf{Offline Policy Learning.} Next, we summarize the common assumptions (from strong to weak) that can yield statistical sample efficiency for policy learning. After that, we introduce extra assumptions for offline learning in the function approximation settings.

\begin{assumption}[Uniform data coverage \citep{yin2021near,ren2021nearly}]\label{assum:uniform}
	For behavior policy, $d_m:=\min_{h,s,a} d_h^\mu (s,a) > 0$. Here the infimum is over all the states satisfying there exists certain policy so that this state can be reached by the current MDP with this policy.\footnote{Note here $d_m$ is defined by minimizing over the state and action spaces. For Assumption~\ref{assum:tope}, $d_m$ only concerns state. } 
\end{assumption}
This is the strongest assumption in offline RL as it requires $\mu$ to explore each state-action pairs with positive probability at different time step $h$. For tabular RL, it mostly holds $1/SA\geq d_m$ under Assumption~\ref{assum:uniform}. This reveals offline learning is generically harder than \emph{the generative model setting} \citep{agarwal2020model,li2020breaking} in the statistical sense. On the other hand, for task where it needs to evaluate different policies simultaneously (such as \emph{uniform OPE} task in \cite{yin2021near}), this is required as the task considered is in general a harder task than offline learning.

\begin{assumption}[Uniform concentrability \cite{szepesvari2005finite,le2019batch,chen2019information,xie2020q}]\label{assum:concen}
	For all policy $\pi$, $C_\mu:=\sup_{\pi,h} ||d^\pi_h(\cdot,\cdot)/d^\mu_h(\cdot,\cdot)||_\infty$. The parameter $C_\mu<+\infty$ is commonly known as ``concentrability efficient''.
\end{assumption}

This is a classical offline RL condition that is commonly assumed in the function approximation scheme (\emph{e.g.} Fitted Q-Iteration in \cite{szepesvari2005finite,le2019batch}, MSBO in \cite{xie2020q}). Qualitatively, this is a uniform data-coverage assumption that is similar to Assumption~\ref{assum:uniform}, but quantitatively, the coefficient $C_\mu$ can be smaller than $1/d_m$ due the $d^\pi_h$ term in the numerator. There are other variants of concentrability efficient \cite{xie2021bellman,nguyen2024sample} that capture the data coverage of behavior policy slightly differently.

\begin{assumption}[Single policy coverage \cite{liu2020off,yin2021towards}]\label{assum:single_concen}
	There exists one optimal policy $\pi^*$, such that $\forall s_h,a_h\in\mathcal{S},\mathcal{A}$, $d^\mu_h(s_h,a_h)>0$ if $d^{\pi^*}_h(s_h,a_h)>0$. We further denote the trackable set as $\mathcal{C}_h:=\{(s_h,a_h):d^\mu_h(s_h,a_h)>0\}$. 
\end{assumption}
 Assumption~\ref{assum:single_concen} is arguably the weakest assumption needed for accurately learning the optimal value $v^*$. It only requires $\mu$ to trace the state-action space of one optimal policy and can be agnostic at other locations.

 \begin{assumption}[Realizability+Bellman Completeness \cite{yin2022offline}]\label{assum:R+BC} The parametric function class $\mathcal{F}$ in Section~\ref{subsec:str} satisfies: 1.
Realizability: for optimal $Q^*_h$, there exists $\theta^*_h\in\Theta$ such that $Q^*_h(\cdot,\cdot)=f(\theta^*_h,\phi(\cdot))$ $\forall h$;
	2. Bellman Completeness: let $\mathcal{G}:= \{V(\cdot)\in\R^\mathcal{S}: s.t.\;\norm{V}_\infty \leq H\}$. Then in this case $\sup_{V\in\mathcal{G}}\inf_{f\in\mathcal{F}}\norm{f-\mathcal{P}_h(V)}_\infty=0$.
\end{assumption} 

Realizability and Bellman Completeness are widely adopted in the offline RL analysis with general function approximations \citep{chen2019information,xie2021bellman}, and they are assumed to ensure class $\mathcal{F}$ is expressive enough to capture the Q-values of the problems and any bounded functions after Bellman updates.


\textit{Additional structural data coverage assumption.} For linear OPL task, we adopt the Assumption~\ref{assume:lope} from linear OPE. As explained before, this assumption is a characterization of Assumption~\ref{assum:uniform} with linear features. 

\begin{assumption}[Linear OPL, Identical to Assumption~\ref{assum:uniform}]\label{assume:lopl}
    Let the population feature covariance $\Sigma_h:=\mathbb{E}_{\mu, h}\left[\phi(s, a) \phi(s, a)^{\top}\right]$. Then we assume $\min_h\lambda_{\text{min}}(\Sigma_h)>0$ with $\lambda_{\text{min}}$ being the minimal eigenvalue.
\end{assumption}

For parametric differentiable function class $\mathcal{F}$, we impose the following structural data coverage assumption to replace Assumption~\ref{assum:uniform}-\ref{assum:single_concen}. The statistical limit is achieved due to properly leveraging the assumptions on the gradient covariance and the quadratic structure. It depends on both the MDPs and the function approximation class $\mathcal{F}$.

\begin{assumption}[Uniform Coverage for $\mathcal{F}$]\label{assum:cover} 
	We assume there exists $\kappa>0$, such that $\forall h\in[H],\theta_1,\theta_2,\theta\in\Theta$, 
  {\small
	\begin{itemize}
	\item\label{assume:eqn1}
$\E_{\mu,h}\left[\left(f(\theta_1,\phi(\cdot,\cdot))-f(\theta_2,\phi(\cdot,\cdot))\right)^2\right]
	\geq \kappa\norm{\theta_1-\theta_2}^2_2, 
	$ 
	\item $
	\E_{\mu,h}\left[\nabla f(\theta,\phi(s,a))\cdot\nabla f(\theta,\phi(s,a))^\top\right]\succ \kappa I, 
	$ 
	\end{itemize}
 }
\end{assumption}

In the linear function approximation regime, Assumption~\ref{assum:cover} reduces to Assumption~\ref{assume:lopl}. The first condition serves more for the ``optimization'' purpose as it can be cast as a variant of the \emph{quadratic growth condition} \cite{anitescu2000degenerate}. For more discussion about this condition, please refer to \cite{yin2022offline,di2023pessimistic}. We will go through the statistical limits of offline policy learning with these assumptions.

\section{Offline Policy Evaluation in Contextual Bandits and Tabular RL}\label{sec:ope}
Let's start by the problem of offline policy evaluation (OPE) --- the problem of evaluating a fixed target policy $\pi$ using data collected by executing a logging (or behavior) policy $\mu$. 

Readers may wonder why this is even a problem.  Admittedly, in \emph{supervised learning}, one can simply evaluate a classifier policy on a validation dataset. Similarly, in \emph{online RL}, one can roll out policy $\pi$ to see how well it works. 

The problem starts to arise in offline problems because we do not have data directly associated with policy $\pi$.

\subsection{OPE in contextual bandits}
Let us build intuition by considering the contextual bandit problem. Contextual bandit (CB) problem is a special case of RL with horizon $H=1$, where the initial state $s$ is referred to as the ``context''. For short horizon problems such as CB, the main challenge is to handle \emph{distribution shift}. Motivated by a change of measure formula
\[
v^\pi_{\text{CB}} = \E_{\substack{s\sim d_1,\\a\sim\pi(\cdot|s)}}[r(s,a)]=\E_{\substack{s\sim d_1,\\a\sim\mu(\cdot|s)}}[\frac{\pi(a|s)}{\mu(a|s)}r(s,a)],
\]
classical methods employ \emph{importance sampling} (IS) \cite{horvitz1952generalization,precup2000eligibility,liu2001monte} to corrects the mismatch in the distributions under the behavior policy $\mu$ and target policy $\pi$. Specifically, let the importance ratio be $\rho:=\pi(a|s)/\mu(a|s)$, then the IS estimator is computed as:
\[
\widehat{v}^\pi_{\text{IS-CB}}=\frac{1}{n}\sum_{i=1}^n \rho^{(i)}r^{(i)}.
\]
It is known to be effective for real-world applications such as news article recommendations \cite{dudik2011doubly,li2011unbiased}.


The mean square estimation error (MSE) of the IS estimator $\widehat{v}^\pi_{\text{IS-CB}}$ decomposes into two terms
\small{
$$
\frac{1}{n}(\E_\mu[\rho(s,a)^2 \Var[r|s,a]] + \Var_\mu[\rho(s,a) \E[r|s,a]]).
$$
}
The first term comes from the noisy reward while the second term comes from the random $(s,a)$ pair. Interestingly, if we make no assumption about $\E[r|s,a]$ and the size of the state-space is large, then IS is minimax optimal \citep[Theorem 1]{wang2017optimal}.  On the contrary, if $\E[r|s,a]$ can be estimated sufficiently accurately, then there are methods that asymptotically do not depend on the $\Var[\E[\cdot]]$. 

Perhaps a bit surprising to some readers, the above conclusion implies that even for \emph{on-policy} evaluation, i.e., $\pi=\mu$ and $\rho \equiv 1$, the naive value estimator of $\frac{1}{n}\sum_i r^{(i)}$ (IS with $\rho\equiv 1$) can be substantially improved using a good reward model. 
 
\subsection{``Curse of Horizon'' in OPE for RL}

The IS estimators are later adopted for long horizon sequential decision making (RL) problems. Concretely, denote the $t$-step importance ratio $\rho_t:=\pi_t(a_t|s_t)/\mu_t(a_t|s_t)$ and the cumulative importance ratio $\rho_{1:t} :=\prod_{t'=1}^t \rho_{t'}$, the (stepwise) Importance Sampling estimators for RL are defined as:
\begin{align*}
\widehat{v}^\pi_{\text{IS}}:=\frac{1}{n}\sum_{i=1}^n \widehat{v}_{\text{IS}}^{(i)}, \quad&\widehat{v}_{\text{IS}}^{(i)}:=\rho_{1:H}^{(i)}\cdot \sum_{t=1}^Hr_t^{(i)};\\
\widehat{v}^\pi_{\text{step-IS}}:=\frac{1}{n}\sum_{i=1}^n \widehat{v}_{\text{step-IS}}^{(i)},\quad &\widehat{v}_{\text{step-IS}}^{(i)}:=\sum_{t=1}^H\rho_{1:t}^{(i)}r_t^{(i)},
\end{align*}
where $\rho_{1:t}^{(i)}=\prod_{t'=1}^t \pi_{t'}(a_{t'}^{(i)}|s_{t'}^{(i)})/\mu_{t'}(a_{t'}^{(i)}|s_{t'}^{(i)})$. In addition, there are many works extend IS estimators and different variants such as \emph{weighted IS estimators} and \emph{doubly robust estimators} \citep{murphy2001marginal,hirano2003efficient,dudik2011doubly,jiang2016doubly} are proposed.

While IS-based OPE methods can correct the distribution shift and are statistically unbiased, the variance of the cumulative importance ratios $\rho_{1:t}$ may grow exponentially as the horizon goes long. We provide two concrete examples in Appendix~\ref{app:CoH_examples} which demonstrate that IS-based methods suffer from exponential variance even in the simplest tabular RL problems.

To make matters worse, the exponential sample complexity in $H$ cannot be improved in general in the large state-space regime unless we make additional assumptions \citep{jiang2016doubly}. This is known as the ``curse of horizon'' in offline RL. We refer readers to a sister article \citep{jiang2024offline} that appears in the same issue of this journal for a quest to obtain sufficient and necessary conditions that enable $\mathrm{poly}(H)$ sample complexity.

Instead of inspecting the exponential separation, we zoom into three well-established sufficient conditions (from Section~\ref{sec:setup}) that circumvent the ``curse of horizon'' and focus on 
providing fine-grained statistical characterization of the optimal OPE error bound and design adaptive estimators that take advantage of individual problem instances. We will cover the case with small finite state spaces in Section~\ref{sec:ope_tabular} and then function approximation in Section~\ref{sec:ope_FA}.





\subsection{OPE in Tabular MDPs}\label{sec:ope_tabular}


The most basic model of interest is the tabular MDP, namely, MDP when the state and action spaces are finite and that the policy $\mu$ gets to visit all states and actions that $\pi$ visits. A statistical lower bound for OPE in the tabular MDP setting is established in \cite{jiang2016doubly}.
\begin{theorem}[Cramer-Rao lower bound for tabular OPE \cite{jiang2016doubly}]\label{thm:lower} For discrete DAG MDPs with horizon $H$, the variance of any unbiased estimator $\hat{v}$ with $n$ trajectories from policy $\mu$ satisfies 
{\small
\begin{align*}
	n \cdot\Var[\hat{v}] \geq \sum_{t=1}^{H}\E_\mu\left[ \frac{d^\pi(s_t,a_t)^2}{d^\mu(s_t,a_t)^2}\Var\Big[V_{t+1}^\pi(s_{t+1})+r_t\Big| s_{t}, a_t\Big]\right].
	\end{align*}
    }
\end{theorem}
The construction of the CR lower bound relies on computing the constrained version of Fisher Information Matrix. Under Assumption~\ref{assum:tope}, this right hand side can be readily bounded by $O(\tau_s\tau_a H^3)$ (after a change of measure into $\E_\pi[\cdot]$)\footnote{The tightest bound is actually $O(\tau_s\tau_a H^2)$ using Lemma~\ref{lem:horizon-re} which we describe later.}, 
which makes the IS estimators with a variance of $\exp(H)$ exponentially suboptimal.  

\noindent\textbf{Marginalized Importance Sampling.} In \cite{xie2019towards,yin2020asymptotically}, we addressed the exponential gap by an idea that is now referred to as Marginalized Importance Sampling. If we re-examine the value objective with RL, by a change of measure formula, 
\[
v^\pi:=\E_{\pi}\left[\sum_{t=1}^H  r_t\right]=\E_{\mu}\left[\sum_{t=1}^H \frac{d^\pi_t(s_t)}{d^\mu_t(s_t)} r_t^\pi(s_t)\right]
\]
with $r_t^\pi(s)=\E_{a\sim \pi(\cdot|s)}[r_t(s,a)|s]$. This reformulation reveals, rather than applying $\rho_{1:t}$, we could instead estimate the marginal state density ratio $d^\pi_t/d^\mu_t$. Inspired by this observation, the \emph{Marginalized Importance Sampling} (MIS) estimator is defined as 
{\small
\begin{equation}\label{MIS_def}
\widehat{v}^\pi_{\text{MIS}} =\frac{1}{n}\sum_{i=1}^n\sum_{t=1}^H\frac{\widehat{d}^\pi_t(s_t^{(i)})}{\widehat{d}^\mu_t(s^{(i)}_t)}\widehat{r}^\pi_t(s^{(i)}).
\end{equation}
}Different design choices for $\widehat{d}^\pi,\widehat{d}^\mu,\widehat{r}^\pi$ in \eqref{MIS_def} yield different MIS estimators.

\noindent\textbf{State MIS (SMIS \cite{xie2019towards}).} For SMIS, $\widehat{d}^\mu_t(\cdot)$ is directly estimated using the empirical mean, \emph{i.e.} $\widehat{d}^\mu_t(s_t):=\frac{1}{n}\sum_i \mathbf{1}(s_t^{(i)}=s_t):=\frac{n_{s_t}}{n}$ whenever $n_{s_t}>0$ and $\widehat{d}^\pi_t(s_t)/\widehat{d}^\mu_t(s_t)=0$ when $n_{s_t}=0$. Marginal state distributions are estimated via recursion $\widehat{d}_t^\pi  = \widehat{P}^{\pi}_t \widehat{d}_{t-1}^\pi$, followed by the estimations $P^\pi_t(s_t|s_{t-1})$ and state reward $r^\pi_t(s_t)$ as:
{\small
\begin{equation}\label{eqn:SMIS-con}
  \begin{aligned}
  \widehat{P}^{\pi}_t(s' | s)  =&  \frac{1}{n_{s}} \sum_{i=1}^{n}   \frac{\pi( a^{(i)}| s)}{\mu( a^{(i)}| s)} \cdot \mathbf{1}\{(s_{t-1}^{(i)},s_t^{(i)}) = (s,s')\};
  \\
  \widehat{r}_t^{\pi}(s)  =&  \frac{1}{n_{s}}\sum_{i=1}^n \frac{\pi(a^{(i)}|s)}{\mu(a^{(i)}|s)} r_t^{(i)} \cdot \mathbf{1}(s_t^{(i)} = s).
  \end{aligned}
  \end{equation}}

SMIS \eqref{MIS_def} explicitly gets rid of the cumulative importance ratio $\rho_{1:t}$ and provides the polynomial sample complexity for horizon under Mean Square Error.

\begin{theorem}
Under Assumption~\ref{assum:tope} and other mild regularity conditions, the MSE of state marginalized importance sampling satisfies
{\small
\begin{align*}
&\mathbb{E}\left[\left( \widehat{v}_{\mathrm{SMIS}}^\pi-v^\pi\right)^2\right]\\ = & \frac{1}{n} \sum_{t=1}^H \E_{\mu} \left[\frac{d_t^\pi\left(s_t\right)^2}{d_t^\mu\left(s_t\right)^2} \operatorname{Var}_\mu[\left.\frac{\pi\left(a_t \mid s_t\right)}{\mu\left(a_t \mid s_t\right)}\left(V_{t+1}^\pi\left(s_{t+1}\right)+r_t\right) \right\rvert\,s_t] \right]\\
& \cdot \big(1+O(\sqrt{\frac{\log n}{n}})\big)+O\big(\frac{1}{n^2}\big).
\end{align*}}
The big $O$ notation hides universal constants. 
\end{theorem}
The MSE of SMIS is $O(\tau_s\tau_a H^3/n)$ and the result holds even when the action space is continuous. This exponentially improves over the standard IS. 

SMIS however, does not match the Cramer-Rao lower bound. In particular, the asymptotic MSE (modulo a $1+O(n^{-1/2})$ multiplicative factor and 
an $O(1/n^2)$ additive factor) is {\small
\[
\frac{1}{n} \sum_{t=1}^H \E_{\mu} \left[\frac{d_t^\pi\left(s_t\right)^2}{d_t^\mu\left(s_t\right)^2} \operatorname{Var}_\mu[\left.\frac{\pi\left(a_t \mid s_t\right)}{\mu\left(a_t \mid s_t\right)}\left(V_{t+1}^\pi\left(s_{t+1}\right)+r_t\right) \right\rvert\,s_t] \right]
\]} and is asymptotically bigger than the CR lower bound in Theorem~\ref{thm:lower} by an additive term 
$$\frac{1}{n} \sum_{t=1}^H \E_{\mu}\bigg[ \frac{d_t^\pi\left(s_t\right)^2}{d_t^\mu\left(s_t\right)^2} \operatorname{Var}_\mu[\frac{\pi_t\left(a_t \mid s_t\right)}{\mu_t\left(a_t \mid s_t\right)} Q_t^\pi\left(s_t, a_t\right)\mid s_t]\bigg]
$$
due to the decomposition via \emph{Law of total variance} that
{\small
\begin{equation}
\begin{aligned}
& \operatorname{Var}_\mu\left[\left.\frac{\pi\left(a_t\mid s_t\right)}{\mu\left(a_t \mid s_t\right)} [V_{t+1}^\pi\left(s_{t+1}\right)+r_t] \right\rvert\, s_t\right] \\
= & \mathbb{E}_\mu\left[\left.\frac{\pi\left(a_t \mid s_t\right)^2}{\mu\left(a_t \mid s_t\right)^2} \operatorname{Var}\left[V_{t+1}^\pi\left(s_{t+1}\right)+r_t \mid s_t, a_t\right] \right\rvert\, s_t\right]\\
+&\operatorname{Var}_\mu\left[\left.\frac{\pi\left(a_t \mid s_t\right)}{\mu\left(a_t \mid s_t\right)} Q_t^\pi\left(s_t, a_t\right) \right\rvert\, s_t\right].
\end{aligned}
\end{equation}}

Not only does it miss the CR lower bound by an additive factor, it also has a worse dependence in horizon $H$. SMIS has an MSE that scales $O(H^3)$, but one can show that the Cramer-Rao lower bound in Theorem~\ref{thm:lower} scales only quadratically in $H$. This is a non-trivial fact that follows from the following lemma (iterative law of total variance).
\begin{lemma}\label{lem:horizon-re}[\cite{gheshlaghi2013minimax,azar2017minimax,yin2020asymptotically}]For any policy $\pi$ and MDP,
{\small
\begin{align*}
&\mathrm{Var}_\pi\left[\sum_{t=1}^H r_t\right] = \sum_{t=1}^H \Big(\E_\pi\left[ \mathrm{Var}\left[V^\pi_{t+1}(s_{t+1})+r_t \middle|s_t,a_t\right] \right]\\
&\quad +  \E_\pi\left[ \mathrm{Var}\left[  \E[V^\pi_{t+1}(s_{t+1})+r_t | s_t, a_t]  \middle|s_t\right] \right]\Big).
\end{align*}
}
\end{lemma}
Observe that the first term on the RHS is the CR lower bound and the second term is non-negative, thus, by $|r_t| \leq 1$ we get that the CR lower bound of $O(\tau_s\tau_a H^2)$.

The gap from the lower bound is rooted in the importance ratios applied for state transition estimations \eqref{eqn:SMIS-con} which eventually propagate into the conditional variance terms of MSE. It is an open problem whether the $O(\tau_s\tau_a H^3/n)$ bound of SMIS can be improved in the setting of (exponentially) large action space $\cA$. Our conjecture is inthe negative, similar to the contextual bandits results by \cite{wang2017optimal}.

When the action space is also finite,  \cite{yin2020asymptotically} proved that an alternative estimator, Tabular MIS, closes the gap. 

\noindent\textbf{Statistically optimal OPE --- Tabular MIS.} 
To remove importance weights \eqref{eqn:SMIS-con}, we need to go beyond state transitions and estimate state-action transitions $\widehat{P}_{t+1}(s'|s,a)$ and state-action reward $\widehat{r}_t(s,a)$ via: 
{\small
\begin{equation}\label{eq:tabular_MIS_construction}
\begin{aligned}
\widehat{P}_{t+1}(s'|s,a)&=\frac{\sum_{i=1}^n\mathbf{1}[(s^{(i)}_{t+1},a^{(i)}_t,s^{(i)}_t)=(s',s,a)]}{n_{s,a}}\\
\widehat{r}_t(s,a)&=\frac{\sum_{i=1}^n r_t^{(i)}\mathbf{1}[(s^{(i)}_t,a^{(i)}_t)=(s,a)]}{n_{s,a}},\\
\end{aligned}
\end{equation}
}with $\widehat{P}_{t+1}(s'|s,a)=0$ and $\widehat{r}_t(s,a)=0$ if $n_{s,a}=0$. The corresponding estimation of $\widehat{P}^\pi_t(s'|s)$ and $\widehat{r}^\pi_t(s)$ are defined by averaging $\widehat{P}_t$ and $\widehat{r}_t$ over $\pi$ and then $\widehat{d}_t^\pi  = \widehat{P}^{\pi}_t \widehat{d}_{t-1}^\pi$.  Tabular MIS (TMIS) \cite{yin2020asymptotically} is then defined via plugging $\widehat{P}^\pi_t$, $\widehat{r}^\pi_t$ and $\widehat{d}^\pi_t$ into \eqref{MIS_def}. It differs from SMIS by leveraging the fact that each state-action pair is visited frequently under the tabular setting.


\textbf{MIS vs. Model-based estimators.} Tabular marginalized importance sampling estimator has dual expressions
{\small
\[
 \frac{1}{n}\sum_{i=1}^n\sum_{t=1}^H\frac{\widehat{d}^\pi_t(s_t^{(i)})}{\widehat{d}^\mu_t(s^{(i)}_t)}\widehat{r}^\pi_t(s^{(i)})=\widehat{v}^\pi_{\text{TMIS}}=\sum_{t=1}^H \sum_{s, a} \widehat{d}_t^\pi\left(s, a\right) \widehat{r}_t\left(s, a\right),
\]}where the right-hand-side expression reveals TMIS is also a model-based estimator as it estimates model transitions $\widehat{P}_t$ and replaces the model with the estimated model for the evaluation purpose. Consequently, despite their differences in complex settings, marginalized importance sampling and model-based estimators can be unified through TMIS in tabular RL, using standard MLE estimators, similar to traditional statistical estimation problems \cite{fisher1925theory}. More importantly, it is statistically optimal for the tabular OPE problem.

\begin{table}[tb]
    \centering
    \begin{tabular}{ccc}
    \hline
         Estimator & MSE (realizable) & MSE (misspecified)  \\\hline
      Import. Sampl. (IS) & $\exp(H)/n$ & $\exp(H)/n$  \\
      Doubly Robust &$\exp(H)/n$ & $\exp(H)/n$ \\
      State MIS &$H^3\tau_s\tau_a/n$ & $H^3\tau_s\tau_a/n+ \mathrm{bias}^2$  \\
      TMIS/FQE/Model-Based  & $H^2\tau_s\tau_a/n$ & $H^2\tau_s\tau_a/n + \mathrm{bias}^2$ \\
      \hline
    \end{tabular}
    \caption{Summary of OPE methods for tabular RL and their squared estimation error.}
    \label{tab:ope_tabular_summary}
\end{table}

\begin{theorem}\label{thm:tmis}Let $\mathcal{D}=\left\lbrace (s_t^{(i)},a_t^{(i)},r_t^{(i)})\right\rbrace_{i\in[n]}^{t\in[H]} $ be obtained by running a behavior policy $\mu$ and $\pi$ is the target policy to evaluate. Under Assumption~\ref{assum:tope} and other mild regularity conditions, the MSE of tabular marginalized importance sampling satisfies
{\small
\begin{equation}
\begin{aligned}
&\mathbb{E}\left[\left( \widehat{v}_{\mathrm{TMIS}}^\pi-v^\pi\right)^2\right]\\ = & \frac{1}{n} \sum_{t=1}^H \E_{\mu} \left[\frac{d_t^\pi\left(s_t,a_t\right)^2}{d_t^\mu\left(s_t,a_t\right)^2} \operatorname{Var}_\mu[\left.\left(V_{t+1}^\pi\left(s_{t+1}\right)+r_t\right) \right\rvert\,s_t] \right]\\
& \cdot [1+O(\sqrt{\frac{\log n}{n}})]+O(\frac{1}{n^2}).
\end{aligned}
\end{equation}}
The big $O$ notation hides universal constants. 
\end{theorem}

\textbf{Asymptotic efficiency and local minimaxity.}\label{remark:asym}
	The error bound implies that 
	{\small$\lim_{n\rightarrow \infty}{n} \cdot\E[ (\widehat{v}_{\mathrm{TMIS}}^\pi -  v^\pi)^2]$} equals  
		{\small 
	\begin{align*}
	\sum_{t=1}^{H}\E_\mu\left[ \frac{d^\pi(s_t,a_t)^2}{d^\mu(s_t,a_t)^2}\Var\Big[V_{t+1}^\pi(s_{t+1})+r_t\Big| s_{t}, a_t\Big]\right].
	\end{align*}
}This result exactly matches the CR-lower bound \ref{thm:lower} and strictly improves \emph{state MIS} estimator, indicating that both the lower bound \ref{thm:lower} and upper bound \ref{thm:tmis} are tight. Modern estimation theory \cite{van2000asymptotic} establishes that CR-lower bound is the asymptotic minimax lower bound for the MSE of \emph{all} estimators in every local neighborhood of the parameter space.\footnote{In classical statistical text, CR-lower bound is often used to lower bound the variance of the class of \emph{unbiased} estimators.} Therefore, Tabular marginalized importance sampling is asymptotically efficient, and locally minimax optimal (i.e. optimal for every problem instance separately).


This result provides new insight even for on-policy evaluation. By default, on-policy evaluation is computed by averaging Monte Carlo returns and its MSE is $\mathrm{Var}_\pi\left[\sum_{t=1}^H r_t\right]$. As implied by Lemma~\ref{lem:horizon-re},
the surprising observation is that TMIS improves the efficiency even for the on-policy evaluation problem. This means the natural Monte Carlo estimator of the reward in the on-policy evaluation problem is in fact asymptotically inefficient. 

\begin{figure}[ht]
    \centering
    \includegraphics[width=0.23\textwidth]{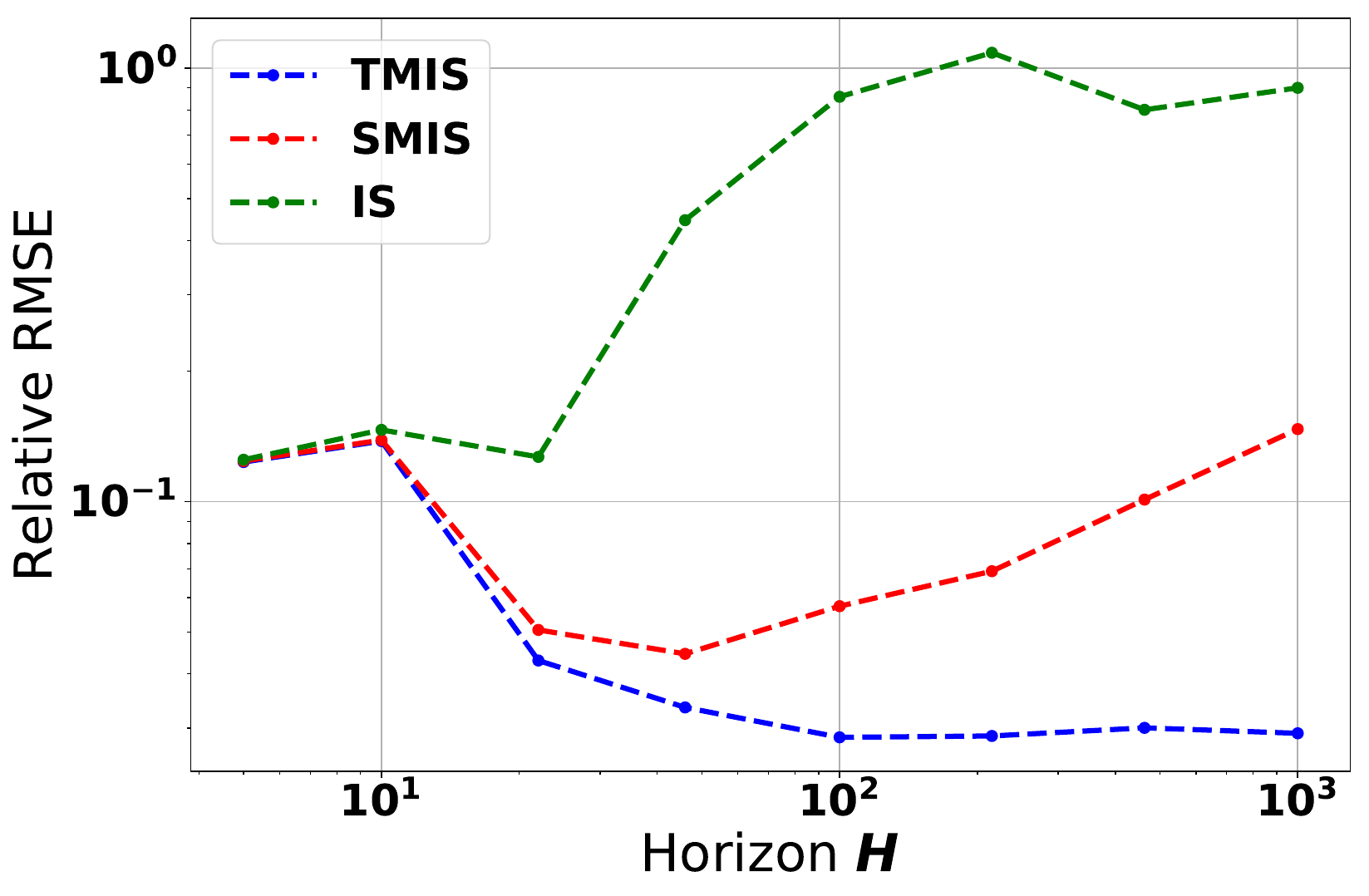}
\includegraphics[width=0.23\textwidth]{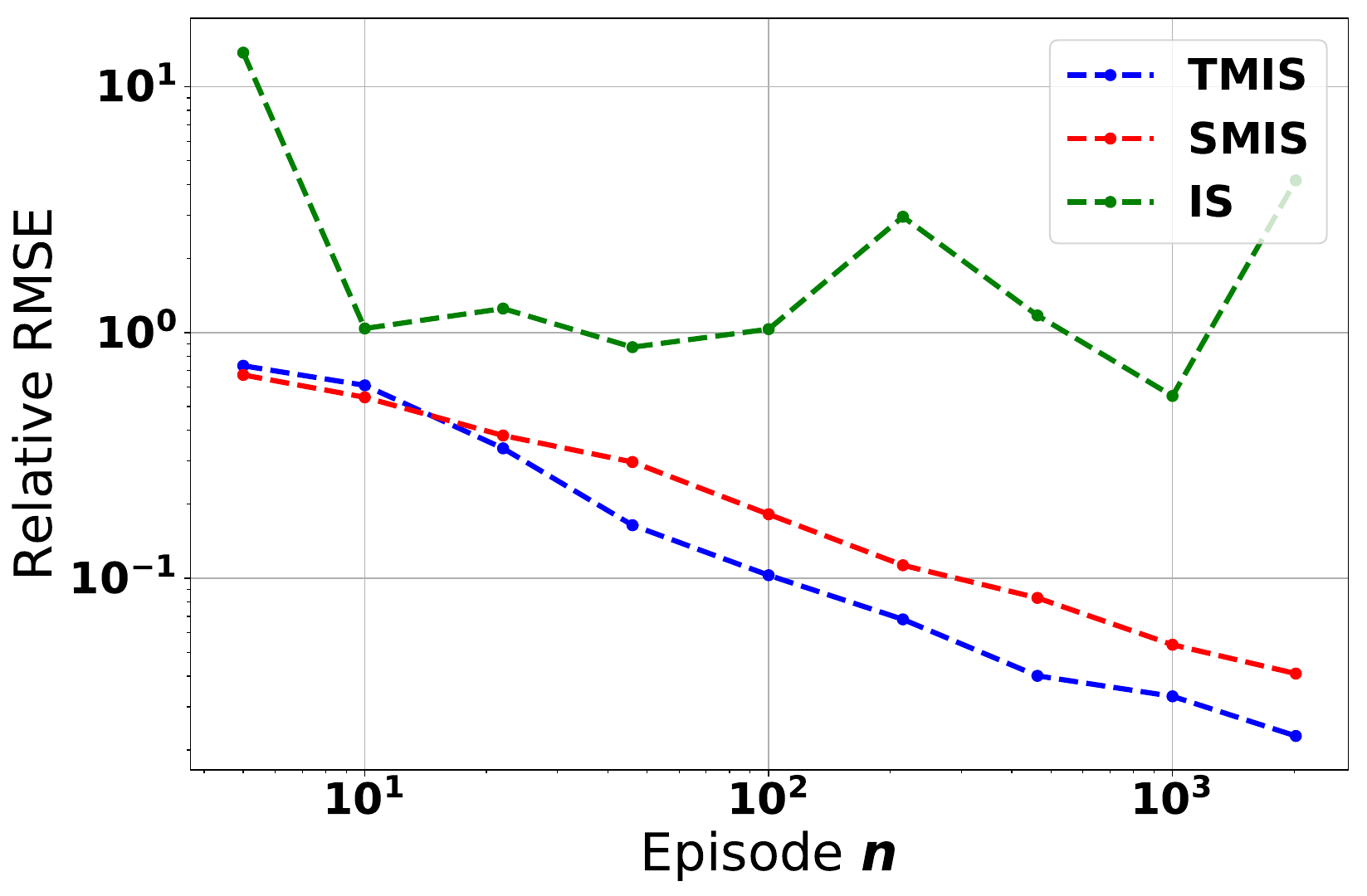}
    \caption{Adopted from \cite{yin2020asymptotically}. Different scaling law for TMIS, SMIS and IS for a time-inhomogenuous MDP. Relative RMSE ($\sqrt{\text{MSE}}/v^\pi$). For episode $n$, the right panel shows both TMIS and SMIS have a convergence rate of $n^{-1/2}$. For horizon $H$, the left panel shows the MSE of TMIS has the optimal dependence $O(H^2)$, while SMIS has the dependence $O(H^3)$.}
    \label{fig:two_figures}
\end{figure}

\textbf{Discussion.} The idea of MIS estimators goes well beyond tabular settings. MIS can be viewed as a dual form of the Bellman value decomposition. This view motivated researchers \citep{liu2018breaking,hallak2017consistent,gelada2019off} to come up with alternative schemes for function approximations in deep RL, e.g., visitation measure $d^\pi(s,a)$ or the importance weights $\rho(s,a) = \frac{d^\pi}{d^\mu}(s,a)$ instead of the value functions. One can also approximates both the $Q$ function and $\rho$ functions, e.g., the double reinforcement learning approach \citep{kallus2020double,kallus2022efficiently} and the DICE family \citep{nachum2019dualdice,uehara2020minimax,zhang2020gendice}. It remains one of the active research areas in RL theory and algorithm design.

As a technical note, the analysis of SMIS and TMIS involves somewhat delicate calculations that leverage the Bellman recursion in both the estimated $\hat{d}^\pi$ and its covariance matrix. As a comparison --- since TMIS is equivalent to the model-based plug-in estimator --- we apply the classical ``simulation lemma'' \citep{kearns2002near} to it, which implies a bound of 
{\small
$$
|\hat{v}^\pi - v^\pi| \leq H^2\sup_{h,s,a}\|\hat{P}_h(\cdot|s,a)-P_h(\cdot|s,a)\|_1 = \tilde{O}(\sqrt{\frac{H^4S^2}{nd_m}}).
$$
}
Observe that our more delicate analysis improves the bound to $\sqrt{\frac{H^2\tau_s\tau_a}{n}}\leq \sqrt{\frac{H^2}{n d_m }}$.


 \section{Offline Policy Evaluation with function approximation}\label{sec:ope_FA}


Next, we switch gears to consider OPE when the (state, action) pairs are described by a continuous feature vector $\phi(s,a) \in \R^d$. This covers most real-life RL problems (such as autonomous driving, robotic arm control, and health care). 

The key challenge here is to generalize across unseen states while maintain the statistical optimality at the same time. In the discrete setting, empirical count (maximum likelihood estimate) is a natural algorithm that is optimal, but it cannot be generalized in the function approximation setting. However, to evaluate MDPs, the Bellman equations are universally true regardless of the setting. As a result, one can apply the approximate dynamic programming principles \cite{powell2007approximate} for the given function class and data. This is realized by the following Fitted Q-Evaluation (FQE). For OPE with function approximation, we review the time-homogenuous RL (\emph{i.e.} transition probabilities are identical across time $P_t=P$) and reformulate offline data $\mathcal{D}=\left\{\left(s_h^k, a_h^k, r_h^k\right)\right\}_{h \in[H], k \in[n]}=\left\{\left(s_i, a_i, r_i\right)\right\}_{i \in[N]}$ throughout the section ($N=nH$).

\vspace{1em}

\textbf{Fitted Q Evaluation.} FQE is a variant of Fitted Q Iteration \cite{ernst2005tree,antos2007fitted} which is designed for policy optimization purpose. A brief history of FQI is discussed in Section~\ref{sec:id_offline_fa}. For a given function class $\mathcal{F}$ and data $\mathcal{D}$, FQE recursively estimates $Q^\pi_h,h\in[H]$ via ($\widehat{Q}^\pi_{H+1}=0$)
\begin{equation}
\label{eqn:fqe}
\widehat{Q}_h^\pi=\underset{f \in \mathcal{F}}{\operatorname{argmin}}\left\{\frac{1}{N} \sum_{i=1}^N\left(f\left(s_i, a_i\right)-y_i\right)^2+\lambda \rho(f)\right\}
\end{equation}
with $y_n=r_i+\int_a \widehat{Q}_{h+1}^\pi\left(s_{i+1}, a\right) \pi\left(a \mid s_{i+1}\right) \mathrm{d} a$. Here $\rho(f)$ is a proper regularizer and is usually chosen as $L_2$, i.e. $\rho(f)=\norm{f}_2^2$. The OPE estimator is 
\[
\widehat{v}^\pi=\mathbb{E}_{s \sim d_1, a \sim \pi(\cdot \mid s)}\left[\widehat{Q}_1^\pi(s, a)\right].
\]
The squared loss function resembles the empirical approximation for Bellman question \eqref{eqn:bellman}.

\subsection{Linear function approximation}\label{sec:ope_linear}
Linear function approximation considers the class $\mathcal{F}_{\text{lin}}=\{f:f(\cdot,\cdot)=\langle \phi(\cdot,\cdot),\theta \rangle,\theta\in \R^d\}$. Denote the shorthand $\phi_n:=\phi(s_n,a_n)$ and $\phi^\pi(s):=\E_{a\sim\pi(\cdot|s)}[ \phi(s, a) ]$, then FQE can be computed recursively via $\widehat{Q}_h(s,a)=\phi(s,a)^\top \widehat{w}^\pi_h$ with 
\[
\widehat{w}^\pi_h=\widehat{R}+\widehat{M}_\pi \widehat{w}_{h+1}^\pi,
\]
where 
$
\widehat{M}_\pi=\widehat{\Sigma}^{-1} \sum_{n=1}^N \phi_n \cdot \phi^\pi\left(s_{n+1}\right)^{\top}, \widehat{\Sigma}=\sum_{n=1}^N \phi_n \phi_n^{\top}+\lambda I_d$ and $
\widehat{R}=\widehat{\Sigma}^{-1} \sum_{n=1}^N r_n \phi_n$. FQE does not learn the model/transition dynamics, and it is generally regraded as a model-free approach. Interestingly, by using the components $\widehat{M}_\pi,\widehat{w}_{h}^\pi$ to approximate the population counterparts ${M}_\pi,{w}_{h}^\pi$, linear FQE is equivalent to the model-based plug-in estimator \cite{duan2020minimax,hao2021bootstrapping}. This phenomenon is similar to TMIS, which can be interpreted as a model-based estimator.

When the data coverage of the behavior policy $\mu$ spans the state-action space and the linear function class is expressive enough, FQE has the following efficiency guarantee. 

\begin{theorem}\label{thm:lope}
Suppose assumption~\ref{assume:lope} (good data coverage) and policy completeness of assumption~\ref{assume:pope} (linear function class is rich enough) are satisfied, then FQE is a consistent OPE estimator with $\sqrt{N}(\widehat{v}^\pi-v^\pi)$ is asymptotically distributed to normal $\mathcal{N}(0,\sigma^2)$. The asymptotic variance is given by 
\[
\begin{aligned}
\sigma^2= & \sum_{h_1,h_2=1}^H\left(\nu_{h_1}^\pi\right)^{\top} \Sigma^{-1} \Omega_{h_1, h_2} \Sigma^{-1} \nu_{h_2}^\pi,
\end{aligned}
\]
where $\nu_h^\pi=\mathbb{E}^\pi\left[\phi\left(s_h, a_h\right) \mid s_1 \sim d_1\right]$, $\Sigma=\frac{1}{H}\sum_{h=1}^H \Sigma_h$, and the cross-covariance 
\[
\Omega_{h_1, h_2}=\mathbb{E}\big[\frac{1}{H} \sum_{h^{\prime}=1}^H \phi\left(s_{h^{\prime}}, a_{h^{\prime}}\right) \phi\left(s_{h^{\prime}}, a_{h^{\prime}}\right)^{\top} \varepsilon_{h_1, h^{\prime}} \varepsilon_{h_2, h^{\prime}}\big]
\]
and $\varepsilon_{h_1, h^{\prime}}=Q_{h_1}^\pi\left(s_{h^{\prime}}, a_{h^{\prime}}\right)-\left(r_{h^{\prime}}+V_{h_1+1}^\pi\left(s_{h^{\prime}+1}\right)\right)$.
\end{theorem}

Critically, the above asymptotic variance is optimal for OPE with linear function approximation. In fact, for linear OPE with Assumption~\ref{assume:lope} and policy completeness, the variance of any unbiased estimator is lower bounded by $\sigma^2$ in Theorem~\ref{thm:lope}. Such a lower bound is constructed via computing the influence function from semi-parametric statistics \cite{van2000asymptotic,tsiatis2006semiparametric}. As a special case, the linear optimality is also consistent with the tabular OPE. For the time-inhomogeneous MDPs, the cross-terms vanish, and the variance $\sigma^2=\sum_{h=1}^H\left(\nu_{h}^\pi\right)^{\top} \Sigma^{-1} \Omega_{h, h} \Sigma^{-1} \nu_{h}^\pi$ coincides with $L_{\text{CR}}$ in Theorem~\ref{thm:lower} when features are indicator functions for states and actions.

\textbf{From offline policy evaluation to offline policy inference.} In addition to the point estimator, Efron's bootstrap \cite{mooney1993bootstrapping} is utilized in literature for distributional inference. By sampling episodes $\mathcal{D}^*=\{\tau_1^*, \ldots, \tau_n^*\}$ independently and with replacement from $\mathcal{D}$, the bootstrapped FQE is
consistent in distribution, meaning $$\sqrt{N}\left(\widehat{v}^\pi_{\text{Bootstrap}}-\widehat{v}^\pi_{\text{FQE}}\right) \xrightarrow{d} \mathcal{N}\left(0, \sigma^2\right).$$ This implies the consistency of the moment estimations, and one particular example is for the second order, i.e. $$\lim_{n\rightarrow\infty} \Var[\sqrt{N}\left(\widehat{v}^\pi_{\text{Bootstrap}}-\widehat{v}^\pi_{\text{FQE}}\right)]=\sigma^2.$$

\textbf{Distribution shift characterization via Minimax-optimal OPE.} The finite sample error bound for FQE provides a similar characterization for the hardness of OPE in the non-asymptotic way. By incorporating the Chi-square divergence $\chi_{\mathcal{F}_{\text{lin}}}^2\left(p, q\right):=\sup _{f \in \mathcal{F}_{\text{lin}}} \frac{\mathbb{E}_{p}[f(x)]^2}{\mathbb{E}_{q}\left[f(x)^2\right]}-1$, there is a simplified finite error bound \cite{duan2020minimax}:
\[
\left|\widehat{v}^\pi-v^\pi\right| \lesssim  H^2 \sqrt{\frac{1+\chi_{\mathcal{F}_{\text{lin}}}^2\left(\pi, {\mu}\right)}{N}}+O\left(N^{-1}\right).
\]This shows the distribution divergence in an explicit way.

\subsection{Parametric function approximation}\label{sec:ope_parametric}

Parametric models extends the linear representation $\langle \phi,\theta\rangle$ to the functional form $f(\theta,\phi)$, allowing for nonlinear or nonconvex structures. Fitted Q-Evaluation over this generic class is less tractable since the regression objective \eqref{eqn:fqe} no longer yield a closed-form solution, and the optimal solution can only be characterized by the optimality condition through the lens of an \emph{Z-estimator} \cite{kosorok2008introduction}
\[
\nabla_\theta\bigg\{\frac{1}{2 N} \sum_{i=1}^N\left[f\left(\widehat{\theta}_h, \phi_i\right)-y_i\left(\widehat{\theta}_{h+1}\right)\right]^2+\lambda \rho(\widehat{\boldsymbol{\theta}})\bigg\}=0,
\]
where $\phi_i=\phi(s_i,a_i)$, {\small$y_j(\theta):=\E_{a'\sim \pi\left(\cdot \mid s_{j+1}\right)} [f\left(\theta, \phi\left(s_{j+1}, a^{\prime}\right)\right)  ]$} and $\widehat{\boldsymbol{\theta}}=(\widehat{\theta}_1,\ldots,\widehat{\theta}_H)$. Yet, FQE is still asymptotically efficient.

\begin{theorem}\label{thm:pope}
Under Assumption~\ref{assume:pope} and mild conditions, when the number of episodes $n\rightarrow\infty$ and $\lambda =o(n^{-1/2})$, we have convergence in distribution ($N=nH$)
$
\sqrt{N}\left(\widehat{v}_\pi-v_\pi\right) \xrightarrow{d} \mathcal{N}\left(0, \sigma^2\right).
$
The asymptotic variance $\sigma^2$ is
\[
\sigma^2= \sum_{h_1, h_2=1}^H [\nu^\pi_{h_1}]^{T} \Sigma_{h_1}^{-1} \Omega_{h_1, h_2} \Sigma_{h_2}^{-1} \nu^\pi_{h_2}.
\]
Here $
\Sigma_h=\mathbb{E}\left[\frac{1}{H} \sum_{j=1}^H\left(\nabla_{\theta_h} f\left(\theta_h^*, \phi_j\right)\right)^{\top}\left(\nabla_{\theta_h} f\left(\theta_h^*, \phi_j\right)\right)\right]
$, $\nu_h^{\pi\top}=\mathbb{E}^\pi\left[\nabla_{\theta_h} f\left(\theta_h^*, \phi\left(s_h, a_h\right)\right) \right]$, the cross covariance is \[\Omega_{i, j}=\mathbb{E}\left[\frac{1}{H} \sum_{h=1}^H\left(\nabla_{\theta_i}^{\top} f\left(\theta_i^*, \phi_h\right)\right)\left(\nabla_{\theta_j} f\left(\theta_j^*, \phi_h\right)\right) \varepsilon_{i, h} \varepsilon_{j, h}\right]\]
with $\varepsilon_{j, h}=f(\theta_j^*, \phi_h)-r_h-\mathbb{E}^\pi[f(\theta_{j+1}^*, \phi_{h+1}) \mid s_{h+1}]$.
\end{theorem}

The parametric FQE strictly subsumes the linear FQE as a special case, and this can be seen by noticing $\nabla_{\theta_j} f(\theta_j^*, \phi_h)=\phi_h$ in the linear case. Besides, there is a matching Cramer Rao lower bound, showing that the asymptotic optimality is achieved \cite{zhang2022off}.

\textbf{On the analysis for OPE with function approximations.} For both linear and parametric cases, the OPE error can be decomposed into two parts $v^\pi-\widehat{v}^\pi=E_1+E_2$, and $E_1=\frac{1}{N}\sum_{i=1}^N$ is the first order term with the form 
\[
e_i:=\sum_{h=1}^H\left(\nu_h^\pi\right)^{\top} \Sigma^{-1} \zeta_i\left(Q_h^\pi\left(s_i, a_i\right)-(r_i^{\prime}+V_{h+1}^\pi(s_i^{\prime}))\right)
\]
and $E_2$ is the higher order term. Depending on the setting, $\zeta_i$ is either $\phi_i$ or $\nabla_\theta f(\theta_h^*,\phi_i)$. Higher order terms are generally handled by the data coverage conditions, and the asymptotic normality can be proved by Martingale CLT \cite{mcleish1974dependent} or Z-Estimator Master Theorem from empirical process theory \cite{kosorok2008introduction}.

\section{Offline Policy Learning in Tabular RL: Pessimism and Instance-Dependent Bounds} \label{sec:OPL}

Policy learning differs from the policy evaluation in that it needs to optimize over a set of policies rather than just evaluating a given policy. Consider the case where there are finite number of policies $\pi_1,\pi_2,\ldots,\pi_K$, and the estimates for the respective policies are $\widehat{v}^{\pi_1},\ldots,\widehat{v}^{\pi_K}$. If that is all the information provided, a natural algorithm for policy learning would be the ERM (Empirical Risk Minimizer)
\begin{equation}\label{eqn:erm}
\widehat{\pi}^*= \mathrm{argmax}_\pi \widehat{v}^\pi.
\end{equation}
However, simply selecting policy via point estimators might not provide the best approach due to uncertainty, and the error in the estimators might cause incorrect prediction about the order of policies. In particular,
the dataset could be biased toward certain states, contain many suboptimal actions, or even contain little information about the optimal policy, which poses significant challenges when trying to generalize beyond the observed data. If an agent is too optimistic in regions where it has little or no data, it may overestimate the value of actions in these regions, leading to poor policy performance. 

The generic recipe for offline decision-making (not just for RL) is the so-called \emph{pessimism in the face of uncertainty}, namely, to stay conservative 
\begin{center}
\textsf{One should be biased towards the more conservative side when the value of a decision is uncertain.}
\end{center}

Pessimism is particularly important in real-world offline applications, where safety and reliability are crucial. For example, in
healthcare, an overly optimistic RL policy might recommend treatments that appear effective in limited data but are unproven or unsafe. Pessimism ensures that decisions are made conservatively, focusing on treatments with scientific evidence.
Besides, offline RL can be used to train self-driving systems using logged data. A pessimistic approach ensures that the vehicle avoids risky maneuvers that haven’t been sufficiently tested in the training data.

Consider the multi-arm bandit (MAB) problem as a simple example, where there are $K$ decision arms. The goal is to identify the arm with the highest mean reward.  Each arm has reward estimate  and uncertainty, then the principle of pessimism will choose 
\begin{equation}\label{eqn:lcb_mab}
\text{argmax}_{k\in[K]}\{\text{Reward\_Estimate}_k-\alpha\cdot \text{Uncertainty}_k\}
\end{equation}
for some penalty parameter $\alpha>0$. The above quantity $$\text{Reward\_Estimate}_k-\alpha\cdot \text{Uncertainty}_k$$ is often termed as the \emph{lower confidence bound}, which discourages the effect of uncertainty when no exploration is allowed (i.e. the offline case). In contrast, online setting optimizes the \emph{upper confidence bound} $\text{Reward\_Estimate}_k+\alpha\cdot \text{Uncertainty}_k$ to encourage exploring region with high uncertainty (see Figure below).

\begin{figure}\label{fig:lcb}
  \begin{center}
\includegraphics[width=0.25\textwidth]{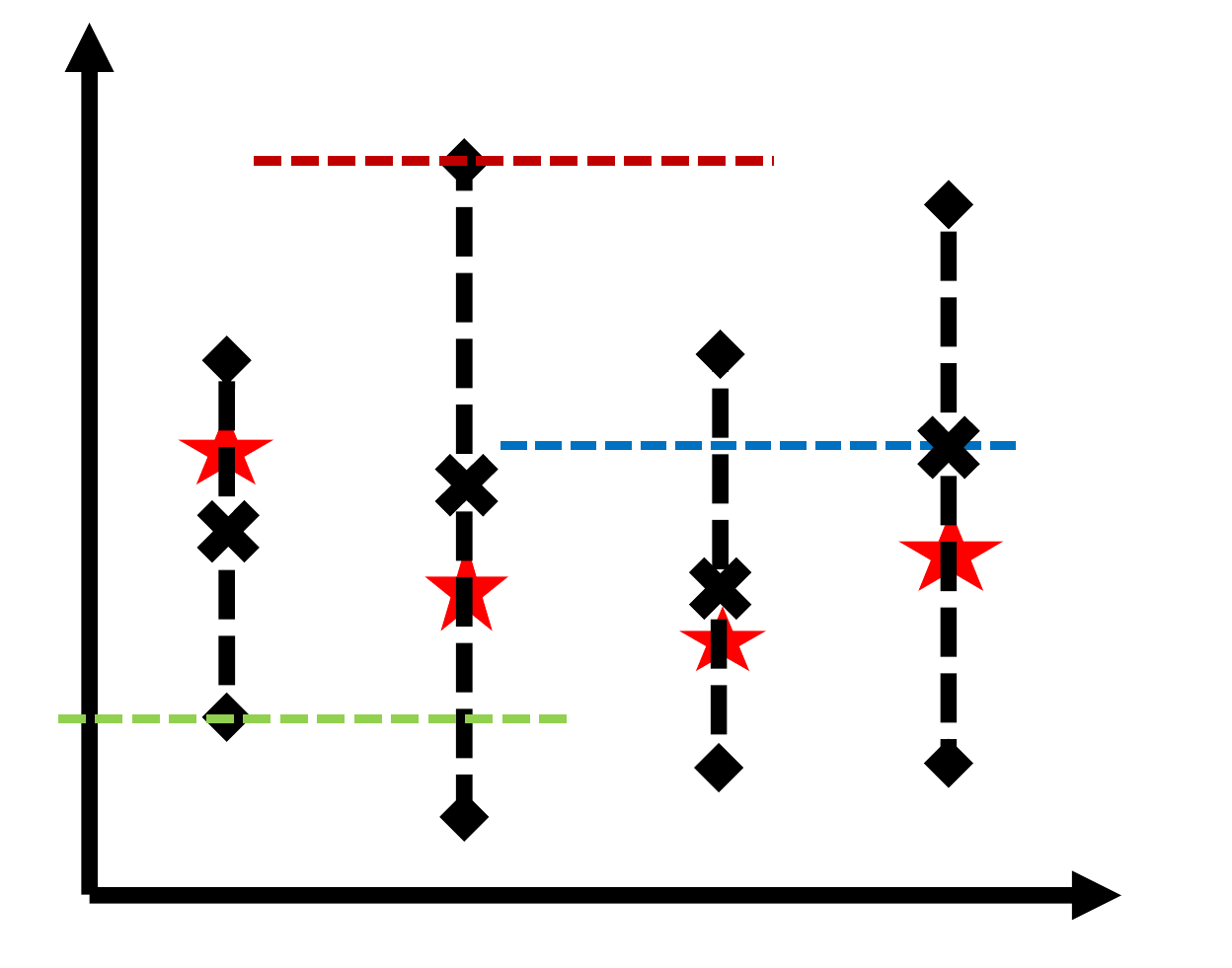}
  \end{center}
  \caption{An instance of MAB problem with pessimism being the right choice. Red choice: upper confidence bound. Blue choice: empirical risk minimizer. Green choice: lower confidence bound \eqref{eqn:lcb_mab}. Red star denotes the true mean reward; black cross denotes the point estimator \eqref{eqn:erm}.}
\end{figure}

\subsection{Pessimism is Minimax Optimal}


An offline policy learning algorithm targets at identifying a policy $\pi_{\text{alg}}$ such that its value differences with respect to the optimal policy $\pi^*$ is small. The performance is theoretically characterized by the \emph{probably approximately correct} (PAC) bound,
\[
v^*-v^{\pi_{\text{alg}}}\leq \mathrm{Poly}(H,S,A,\frac{1}{\sqrt{n}},C_\mu) \text{ with high probability,}
\]   
meaning the performance gap is polynomial in the planning horizon, number of states and actions, $1/\sqrt{n}$, and certain data coverage parameter $C_\mu$. As the number of the episodes $n$ goes to infinity, $v^*-v^{\pi_{\text{alg}}}\rightarrow 0$. 

In addition, for the given $n$ episodic data, the minimax risk (suboptimality gap) \cite{wainwright2019high,minimax} 
\[
\mathcal{R}_n:=\inf _{\pi_{\text{alg}}} \sup _{\text{MDP} \;\mathcal{M}} \mathbb{E}_{\mathcal{M}}\left[v^*-v^{\pi_{\text{alg}}}\right]
\]
measures the best possible performance (information-theoretical limit) for a class of MDP problems $\mathcal{M}$ in the worst-case-scenario sense. If for certain algorithm $\mathcal{A}$, the PAC bound of its suboptimality gap $v^*-v^{\pi_{\text{alg}}}$ matches $\mathcal{R}_n$, then we call algorithm $\mathcal{A}$ minimax optimal. The key feature for minimax lower bound is that the supremum is taken over the whole MDP class, making it instance independent. This is to say, the worst case optimality are optimal ``globally''. 

\begin{itemize}[leftmargin=*]
    \item For the Uniform data coverage $d_m>0$ (Assumption~\ref{assum:uniform}), the minimax optimal bound has rate $\Theta(\sqrt{\frac{H^3}{n\cdot d_m}})$, and it is attained by choosing the ERM estimator \cite{yin2021near,ren2021nearly,yin2021near-optimal}. 
    \item For the single policy coverage $C^*=\norm{\frac{d^{\pi^*}_h}{d^\mu_h}}_\infty$ (Assumption~\ref{assum:single_concen}), the minimax optimal bound has rate $\Theta(\sqrt{\frac{H^3SC^*}{n}})$ by using the reference advantage techniques \cite{xie2021policy,rashidinejad2021bridging}. 
\end{itemize}

 The (near-)optimal worst-case performance bounds that depend on their data-coverage coefficients are valuable as they do not depend on the structure of the particular problem, therefore, remain valid even for pathological MDPs. However, the global optimal characterizations are unable to depict what types of decision processes and what kinds of behavior policies are inherently easier or more challenging for offline RL. In particular, the empirical performances of real applications are often far better than what those non-adaptive / problem-independent bounds would indicate. For example, city driving vs. highway driving in autonomous driving. In both cases, the state-action space could be defined by the position, velocity, orientation of the vehicle, but city driving is more complex due to a highly dynamic and unpredictable environment whereas highway driving is simpler in comparison because the environment is more structured.

Alternatively, rather than obtaining the PAC bound that depends on the global parameters $H,S,A$, we can delve into the instance level and consider the instance-dependent characterization via transition kernel $P$, reward $r$, and behavior policy $\mu$. In general, instance dependent bounds should have the following properties:  

\begin{itemize}[leftmargin=*]
    \item It adapts to the individual instances and only require minimal assumptions so they can be widely applied in most cases.
    \item  It should characterize the system structures of the specific problems, hold even for peculiar instances that do not satisfy the standard data-coverage assumptions. 
    \item It should recover the worst-case guarantees when the data-coverage assumptions are satisfied.
\end{itemize}
For the rest of the section, we review how these guidelines are accomplished for the policy learning tasks. 

\subsection{Towards instance optimality via Pessimism}\label{sec:tabular_instance_optimality}
In this section, we review instance-dependent offline RL in the tabular setting.

\textbf{Pessimistic Value Iteration.} To approximate the optimal Q-function in \eqref{eqn:bellman}, one can perform the update 
\begin{equation}\label{eqn:VI}
\widehat{Q}_h(\cdot,\cdot)\leftarrow (\widehat{\mathcal{P}}_{h}\widehat{Q}_{h+1})(\cdot,\cdot),
\end{equation}
where $\widehat{\mathcal{P}}_{h}$ denotes the approximation for the Bellman operator in \eqref{eqn:operator} and $\widehat{V}_h(s):=\max_a \widehat{Q}_h(s,a)$. For a uncertainty quantification $\Gamma_h(\cdot,\cdot)$, the principle of pessimism applies 
\begin{equation}\label{eqn:UQ}
\widehat{Q}_h(\cdot,\cdot)\leftarrow \widehat{Q}_h(\cdot,\cdot)-\Gamma_h(\cdot,\cdot).
\end{equation}
This is the analogy to the lower confidence bound in the bandit setting \eqref{eqn:lcb_mab}. The learned policy and value function by pessimistic value iteration are defined as ($\forall h\in[H]$): 
\begin{align*}
{\pi}_{h}^{\text{PVI}}(\cdot|s_h)\leftarrow &\mathrm{argmax}_{\pi_h}\langle \widehat{Q}_h(s_h,\cdot),\pi_h(\cdot|s_h)\rangle.\\
\widehat{V}_h(s_h)\leftarrow & \mathrm{max}_{\pi_h}\langle \widehat{Q}_h(s_h,\cdot),\pi_h(\cdot|s_h)\rangle.
\end{align*}
\textbf{Model-based Estimators.} Recall 
the approximated Bellman operator $\widehat{\mathcal{P}}_h$ in \eqref{eqn:VI} is defined via the plug-in estimators:\footnote{$n_{s,a,h}:=\sum_{\tau=1}^n\mathbf{1}[s_h^{\tau},a_h^{\tau}=s,a]$ be the total counts that visit $(s,a)$ pair at time $h$.} 
{
\begin{align*}\label{eqn:mb_est}
\widehat{P}_h(s'|s,a)=\frac{\sum_{\tau=1}^n\mathbf{1}[(s^{\tau}_{h+1},a^{\tau}_h,s^{\tau}_h)=(s^\prime,s,a)]}{n_{s,a}},\\
\widehat{r}_h(s,a)=\frac{\sum_{\tau=1}^n\mathbf{1}[(a^{\tau}_h,s^{\tau}_h)=(s,a)]\cdot r_h^\tau}{n_{s,a}}.
\end{align*}}If $n_{s,a}=0$, $\widehat{P}_h(s'|s,a)={1}/{S},\widehat{r}_h(s,a)=0$. 

\vspace{1em}

\textbf{A Bernstein-style uncertainty.} It turns out that the following variance-dependent quantity 
{\small
 \[
\Gamma_h(s,a)=\widetilde{O}\bigg[\sqrt{\frac{\mathrm{Var}_{\widehat{P}_{s,a}}(\widehat{r}_h+\widehat{V}_{h+1})}{n_{s,a}}}+\frac{H}{n_{s,a}}\bigg]
\]
}describes the uncertainty for \eqref{eqn:UQ} appropriately. The conditional variance ${\mathrm{Var}_{\widehat{P}_{s,a}}(\widehat{r}_h+\widehat{V}_{h+1})}$ corresponds to the \textbf{aleatoric uncertainty} which measures the intrinsic uncertainty of the transition $P$, given the state-action $(s,a)$. This is the uncertainty due to the natural variability of the system being modeled, which cannot be reduced by collecting more data. The term $1/n_{s,a}$ is the \textbf{epistemic uncertainty} that comes from incomplete knowledge and can be reduced by gathering more data \cite{hullermeier2021aleatoric}.

The condition variance in $\Gamma_h$ creates the Bernstein-style pessimism. Compared to the Hoeffding-style pessimism $\widetilde{O}(H/\sqrt{n_{s,a}})$ which is overly pessimistic (due to $\sqrt{\mathrm{Var}_{\widehat{P}}(\widehat{r}_h+\widehat{V}_{h+1})}\leq H$), $\Gamma_h$ is more data-adaptive. Furthermore, for the fully deterministic environments where the transitions and rewards are deterministic, the conditional variances vanishes and the uncertainty has a faster scale $1/n_{s,a}$. 

\begin{theorem}\label{thm:APVI}
	Under the Assumption~\ref{assum:single_concen}, denote $\bar{d}_m:=\min_{h\in[H]}\{d^\mu_h(s_h,a_h):d^\mu_h(s_h,a_h)>0\}$. For any $0<\delta<1$, such that when $n> 1/\bar{d}_m\cdot\log(HSA/\delta)$, with probability $1-\delta$, the output of Pessimistic Value Iteration satisfies 
	\begin{equation}\label{eqn:APVI}
    \begin{aligned}
	& 0\leq  v^*-v^{{\pi}^{\text{PVI}}}\\
 \lesssim & \sum_{h=1}^H\sum_{(s,a)\in\mathcal{C}_h}d^{\pi^*}_h(s,a)\cdot\sqrt{\frac{\mathrm{Var}_{P_{s,a}}(r_h+V^*_{h+1})}{ n\cdot d^\mu_h{(s,a)}}}\\
 +&\widetilde{O}\left(\frac{H^3}{n\cdot \bar{d}_m}\right).
    \end{aligned}
	\end{equation}
\end{theorem}

Unlike the worst-case bounds that rely on the data-coverage parameters, the instance bound requires the minimal assumption~\ref{assum:single_concen}. The key distinction is the main term in \eqref{eqn:APVI} is expressed by the system quantities that admits no explicit dependence on $H,S,A$. It depicts the interrelations within the problem when the problem instance is a tuple $(\mathcal{M},\pi^*,\mu)$: an MDP $\mathcal{M}$ (coupled with the optimal policy $\pi^*$) with the data rolling from an offline behavior policy $\mu$, so it helps understand what type of problems are harder / easier than others in a \emph{quantitative} way. 

The complexity of the main term in \eqref{eqn:APVI} can be decomposed into $\frac{d^{\pi^*}_h(s,a)}{\sqrt{d^\mu_h{(s,a)}}}\cdot \sqrt{\mathrm{Var}_{P_{s,a}}(r_h+V^*_{h+1})}$, and this reveals the learning hardness of offline RL stems from two aspects. 

\begin{itemize}[leftmargin=*]
    \item \textbf{Environmental variation}\footnote{\cite{maillard2014hard} named a similar quantity environmental norm.} $\sqrt{\mathrm{Var}_{P_{s,a}}(r_h+V^*_{h+1})}$ is jointly determined by the stochasticity of transition, reward, and optimal value function. A problem with lower environmental variation is easier than a problem with higher environmental variation. This theoretical characterization explains the intuition that stochastic environments are generally harder than the deterministic environments.
    
    \item \textbf{Distribution mismatch} $\frac{d^{\pi^*}_h(s,a)}{\sqrt{d^\mu_h{(s,a)}}}$ is the other factor that affects the learning hardness. When the behavior policy $\mu$ deviates far from the optimal policy $\pi^*$, the problem is intrinsically harder since the mismatch ratio becomes large. When $\pi^*=\mu$, the mismatch ratio is bounded by $1$ for all states and actions. 
\end{itemize}
Due to its fine-grained expression, \eqref{eqn:APVI} is named \emph{Intrinsic offline learning bound} by recent literature \cite{yin2021towards}. It also subsumes the existing worst-case bounds that are optimal.

For uniform data-coverage \ref{assum:uniform}, the optimal suboptimality is $\Theta(\sqrt{\frac{H^3}{nd_m}})$. The intrinsic RL bound can be upper bounded by this rate via \emph{cauchy inequality} and Lemma~\ref{lem:horizon-re}:\footnote{Here $\odot$ denotes element-wise multiplication.}
{\small
\begin{align*}
&v^*-v^{{\pi}^{\text{PVI}}}\lesssim\sum_{h=1}^H\langle d^{\pi^*}_h(\cdot),\sqrt{\frac{\mathrm{Var}_{P_{(\cdot)}}(r_h+V^*_{h+1})}{ n\cdot d^\mu_h{(\cdot)}}}\rangle \\
=& \sum_{h=1}^H\langle \sqrt{d^{\pi^*}_h(\cdot)},\sqrt{\frac{d^{\pi^*}_h(\cdot)\odot\mathrm{Var}_{P_{(\cdot)}}(r_h+V^*_{h+1})}{ n\cdot d_m}}\rangle\\
\leq &\sum_{h=1}^H\norm{ \sqrt{d^{\pi^*}_h(\cdot)}}_2\norm{\sqrt{\frac{d^{\pi^*}_h(\cdot)\odot\mathrm{Var}_{P_{(\cdot)}}(r_h+V^*_{h+1})}{ n\cdot d_m}}}_2\\
\leq& \sqrt{\frac{H\cdot \mathrm{Var}_{\pi^*}(\sum_{h=1}^{H} r_{h})}{n\cdot d_m}}\leq \sqrt{\frac{H^3}{n\cdot d_m}}
\end{align*}
}which recovers the optimal rate.

For the single policy coverage \ref{assum:single_concen} with $C^*:=\norm{\frac{d^{\pi^*}_h}{d^\mu_h}}_\infty<\infty$, a similar computation using Lemma~\ref{lem:horizon-re} can recover the optimal rate $\Theta(\sqrt{\frac{H^3SC^*}{n }})$ via
{\small
\begin{align*}
&v^*-v^{{\pi}^{\text{PVI}}}\lesssim\sum_{h=1}^H\langle d^{\pi^*}_h(\cdot),\sqrt{\frac{\mathrm{Var}_{P_{(\cdot)}}(r_h+V^*_{h+1})}{ n\cdot d^\mu_h{(\cdot)}}}\rangle \\
\leq&\sqrt{\frac{C^*}{n}} \sum_{h=1}^H\langle \sqrt{d^{\pi^*}_h(\cdot)},\sqrt{\mathrm{Var}_{P_{(\cdot)}}(r_h+V^*_{h+1})}\rangle\\
\leq & \sqrt{\frac{SC^*}{n }}\sum_{h=1}^H\sqrt{{\sum_{s\in\mathcal{S}}d^{\pi^*}_h(s,\pi^*_h(s))\cdot \mathrm{Var}_{P_{s,\pi^*_h(s)}}(r_h+V^*_{h+1})}}\\
\leq& \sqrt{\frac{SC^*}{n }} \sqrt{H}\cdot\sqrt{\mathrm{Var}_{\pi}\left[\sum_{t=1}^{H} r_{t}\right]}\leq \sqrt{\frac{H^3SC^*}{n }}.
\end{align*}
}

\textbf{Problem dependent domain.} 
Similar to the online RL \cite{zanette2019tighter}, if we denote  $\mathbb{Q}^*_h=\max_{s_h,a_h}\mathrm{Var}_{P_{s_h,a_h}}(r_h+V^*_{h+1})$ for all $h\in[H]$, and relax the total sum of rewards to be bounded by any arbitrary value $\mathcal{B}$ (\emph{i.e.} $\sum_{h=1}^H r_h\leq \mathcal{B}$), then Theorem~\ref{thm:APVI} implies:
	{\small
	\[
	v^*-v^{{\pi}^{\text{PVI}}}\leq\min\bigg\{\widetilde{O}\big(\sum_{h=1}^H\sqrt{\frac{\mathbb{Q}^*_h}{n\bar{d}_m}}\big),\widetilde{O}\big(\sqrt{\frac{H \mathcal{B}^2}{n\bar{d}_m}}\big)\bigg\}+\widetilde{O}(\frac{H^3}{n\bar{d}_m}).
	\]}  
For the problem instances with either small $\mathcal{B}$ or small $\mathbb{Q}_h^*$, the intrinsic bound yields much better performances, as discussed in the following.

\emph{Deterministic systems.} For systems equipped with low stochasticity, \emph{e.g.} robotics, or even deterministic dynamics, \emph{e.g.} the game of GO, the agent needs less experience for each state-action therefore the learning procedure could be much faster. In particular, when the system is fully deterministic (in both transitions and rewards) then $\mathbb{Q}^*_h=0$ for all $h$. This enables a faster convergence rate of order $\frac{H^3}{n\bar{d}_m}$ and significantly improves over the existing worst-case results that have order $\frac{1}{\sqrt{n}}$. 

\emph{Partially deterministic systems.} Sometimes, practical applications can have a mixture model which contains both deterministic and stochastic steps. In those scenarios, the main complexity is decided by the number of stochastic stages: suppose there are $t$ stochastic $P_h,r_h$'s and $H-t$ deterministic $P_{h'},r_{h'}$'s, then completing the offline learning guarantees {\small$t\cdot\sqrt{{\max Q^*_h}/{n \bar{d}_m}}$} suboptimality gap, which could be much smaller than {\small$H\cdot\sqrt{{\max Q^*_h}/{n \bar{d}_m}}$} when $t\ll H$. 

\emph{Fast mixing domains.} Consider a class of highly mixing non-stationary MDPs that satisfies the transition $P_h(\cdot|s_h,a_h):=\nu_h(\cdot)$ depends on neither the state $s_h$ nor the action $a_h$. Define $\bar{s}_{t} := \arg \max V_{t}^{*}(s)$ and $\underline{s}_{t} := \arg \min V_{t}^{*}(s)$. Also, denote $\mathrm{rng}V^*_h$ to be the range of $V^*_h$.
In such cases, Bellman optimality equations have the form
{
\begin{align*}
&V_{h}^{*}\left(\bar{s}_{h}\right)=\max _{a}\left(r_h\left(\bar{s}_{h}, a\right)+\nu_h^{\top} V_{h+1}^{*}\right),\\
&V_{h}^{*}\left(\underline{s}_{h}\right)=\max _{a}\left(r_h\left(\underline{s}_{h}, a\right)+\nu_h^{\top} V_{h+1}^{*}\right),
\end{align*}
}which yields $\mathrm{rng}V^*_h=V_{h}^{*}\left(\bar{s}_{h}\right)-V_{h}^{*}\left(\underline{s}_{h}\right)=\max_ar_h\left(\bar{s}_{h}, a\right)-\min_ar_h\left(\underline{s}_{h}, a\right)\leq 1$, and this in turn gives $\mathbb{Q}_h^*\leq 1+(\mathrm{rng}V^*_h)^2=2$. As a result, the suboptimality is bounded by $\widetilde{O}(\sqrt{H^2/nd_m})$ in the worst case. This reveals, the class of non-stationary fast mixing MDPs is only as hard as the family of stationary MDPs in the minimax sense ($\Omega(H^2/d_m\epsilon^2)$).

\subsection{Assumption-Free Offline RL}\label{sec:AF-offlineRL}

We now review the scenario where the behavior policy can be arbitrary in this section. In this case, $\mu$ might not cover any optimal policy $\pi^*$ (\emph{i.e.} there might be high reward location $(s,a)$ that $\mu$ can never visit). This can happen when a mediocre doctor only uses one treatment for certain patient all the time. Statistically, even with the infinite amount of episodic data, algorithms might not learn the optimal policy exactly. 

To better characterize the discrepancy, an augmented MDP $\mathcal{M}^\dagger$ is defined with one extra state $s_h^\dagger$ for all $h\in\{2,\ldots,H+1\}$ with the augmented state space $\mathcal{S}^\dagger=\mathcal{S}\cup\{s^\dagger_h\}$. Compared to the original MDP $\mathcal{M}$, the transition and the reward are modified as follows: 
{\small
	\begin{align*}
	P^{\dagger}_h(\cdot \mid s_h, a_h)&=\left\{\begin{array}{ll}
	P_h(\cdot \mid s_h, a_h), \;n_{s_h,a_h}>0, \\
	\delta_{s^{\dagger}_{h+1}}, \; s_h=s_h^{\dagger} \text { or } n_{s_h,a_h}=0.
	\end{array}\right. \\ r^{\dagger}( s_h, a_h)&=\left\{\begin{array}{ll}
	r(s_h, a_h), \; n_{s_h,a_h}>0, \\
	0, \; s_h=s^{\dagger}_{h} \text { or } n_{s_h,a_h}=0.
	\end{array}\right.
	\end{align*}}
 here $\delta_s$ is the Dirac measure and we denote $V^{\dagger \pi}_h$ and $v^{\dagger\pi}$ to be the values under $\mathcal{M}^\dagger$. In this case, the pessimistic value iteration guarantees with high probability that, for any behavior policy $\mu$,{\small
 \begin{align*}
	 &v^*-v^{{\pi}^{\text{PVI}}}\lesssim \sum_{h=2}^{H+1}d^{\dagger\pi^*}_h(s^\dagger_h)\\
  +&\sum_{h=1}^H\sum_{(s,a)\in\mathcal{C}_h}d^{\dagger\pi^*}_h(s,a)\cdot\sqrt{\frac{\mathrm{Var}_{P^\dagger_{s,a}}(r^\dagger_h+V^{\dagger\pi^*}_{h+1})}{ n\cdot d^\mu_h{(s,a)}}}\\
  +&\widetilde{O}\left(\frac{H^3}{n\bar{d}_m}\right).
  \end{align*}
}The off-support gap $\sum_{h=2}^{H+1}d^{\dagger\pi^*}_h(s^\dagger_h)$ satisfies $d^{\dagger\pi^*}_h(s^\dagger_h)=\sum_{t=1}^{h-1}\sum_{(s,a)\in\mathcal{S}\times\mathcal{A}\backslash\mathcal{C}_t}d^{\dagger\pi^*}_t(s,a)$ and
$\mathcal{C}_h:=\{(s,a):d^\mu_h(s,a)>0\}$. When assumption~\ref{assum:uniform} or \ref{assum:single_concen} is satisfied, this gap vanishes since $\mathcal{S}\times\mathcal{A}\backslash \mathcal{C}_h=\emptyset$, and the assumption-free generalization reduces to \eqref{eqn:APVI}.

Beyond the tabular setting, assumption-free RL is also considered in the function approximation setting. \cite{liu2020provably} uses $\epsilon_\zeta$, the probability under a policy of escaping to state-actions with insufficient data during an episode, to measure the state-action region that is agnostic to the behavior policy, then it incurs an off-support gap $\frac{V_{\max}\epsilon_\zeta}{1-\gamma}$.\footnote{In the discounted setting, $1/(1-\gamma)$, the effective horizon, is similar to $H$ in the finite horizon setting. }
The other study relies the condition \emph{Compliance of Dataset} which only requires the data tuples $(s_i,a_i,r_i,s'_i)$ to follow the same MDP transition $P$ that might not cover any good policy, and the data agnostic region is handled by regularization to avoid singularity \cite{jin2021pessimism}. For general function approximation, the off-support gap is characterized by Theorem 3.1 of \cite{xie2021bellman}.

\section{Offline Policy Learning with Function Approximations}\label{sec:id_offline_fa}

Fitted Q-Iteration (FQI) \cite{ernst2005tree}, which is initially named as \emph{fitted value iteration} (FVI) \cite{gordon1999approximate}, makes it possible to take full advantage of any regression algorithm
for achieving generalization for reinforcement learning. In particular, it is widely adopted for offline RL when only historical data are provided \cite{antos2007fitted,munos2008finite}. In the previous section~\ref{sec:ope_FA}, we have seen that its variants Fitted Q-Evaluation are the statistically optimal estimator for the \emph{offline policy evaluation} task. For policy learning with function approximation, we review FQI/FVI as it still yields strong instance-dependent guarantees.

\textbf{Pessimism remains effective for function approximation.} The general prototype of pessimism combined with FVI in \eqref{eqn:VI}, \eqref{eqn:UQ} remains valid for any MDPs. If the point-wise condition \cite{jin2021pessimism}
\[
 |(\widehat{\mathcal{P}}_{h}\widehat{Q}_{h+1})(\cdot,\cdot) - ({\mathcal{P}}_{h}\widehat{Q}_{h+1})(\cdot,\cdot)|\leq \Gamma(\cdot,\cdot)
\]
holds true, then the suboptimality gap can be bounded by 
\[
v^*-v^{{\pi}^{\text{PFVI}}}\leq 2 \sum_{h=1}^H \mathbb{E}_{(s_h,a_h)\sim \pi^*}\left[\Gamma_h\left(s_h, a_h\right)\right]. 
\]
\subsection{OPL with Linear Function Approximation}\label{sec:OPL_linear}  

When Linear MDP models (c.f. section~\ref{subsec:str}) are instantiated, the FVI solves 
{\small
\begin{equation}\label{eqn:FVI}
\begin{aligned}
\widehat{w}_h=&\text{argmax}_{w\in\R^d}\bigg\{\sum_{\tau=1}^n\left(r_h^\tau+\widehat{V}_{h+1}\left(x_{h+1}^\tau\right)-\phi\left(x_h^\tau, a_h^\tau\right)^{\top} w\right)^2\\
+&\lambda \norm{w}_2^2\bigg\},
\end{aligned}
\end{equation}}
and $\widehat{\mathcal{P}}_{h}\widehat{Q}_{h+1}(\cdot,\cdot)=\phi(\cdot,\cdot)^\top \widehat{w}_h$ has a closed-form solution. The pessimism $\Gamma_h(s,a)=dH\sqrt{\phi(s,a)^\top \Lambda_h^{-1}\phi(s,a)}$, and ${\phi(s,a)^\top \Lambda_h^{-1}\phi(s,a)}$ represents the effective number of samples observed in offline data along the $\phi$ direction, and thus represents the uncertainty along the $\phi$ direction. Here $\Lambda_h=\sum_{\tau=1}^n \phi\left(x_h^\tau, a_h^\tau\right) \phi\left(x_h^\tau, a_h^\tau\right)^{\top}+\lambda \cdot I$ is the Gram matrix. The resulting bound scales as 
{\small
\begin{equation}\label{eqn:PFVI}
v^*-v^{{\pi}^{\text{PFVI}}}\lesssim dH \sum_{h=1}^H \mathop{\mathbb{E}}_{(s_h,a_h)\sim \pi^*}\left[\sqrt{\phi(s_h,a_h)^\top \Lambda_h^{-1}\phi(s_h,a_h)}\right].
\end{equation}
}

\textbf{Is FQI/FVI itself sufficient for optimality?} When reducing to the tabular MDPs with $\phi(s,a)=\mathbf{1}_{s,a}$, PFVI has the form {\small$\widetilde{O}(d H\cdot\sum_{h,s,a}d^{\pi^*}_h(s,a)\sqrt{\frac{1}{n\cdot d^\mu_h(s,a)}})$}, and this deviates from Theorem~\ref{thm:APVI} {\small$\widetilde{O}(\sum_{h,s,a}d^{\pi^*}_h(s,a)\sqrt{\frac{\Var_{P_{s,a}}(r+V^*_{h+1})}{n\cdot d^\mu_h(s,a)}})$} by a factor of $H^{1/2}$.
By direct comparison, it can be seen that PFVI cannot get rid of the explicit $H$ factor due to missing the variance information (\emph{w.r.t} $V^*$). 

Intuitively, it might not be ideal to put equal weights on all the training samples in the FQI/FVI objectives, as different data pieces carry different ``amount'' of information. The term $\Var_{P_{s,a}}(r+V^*_{h+1})$ happens to measure the aleatoric uncertainty at location $(s,a)$. If $\Var_{P_{s_1,a_1}}(r+V^*_{h+1})\ll \Var_{P_{s_2,a_2}}(r+V^*_{h+1})$, then the information contained in sample piece $(s_1,a_1,s'_1,r_1)$ is more certain than the sample $(s_2,a_2,s'_2,r_2)$. To address this, existing literature deployed variance reweighting \cite{min2021variance,yin2022near,xiong2022nearly} for FQI/FVI.

\textbf{Variance-weighted FVI.} Instead of regressing via \eqref{eqn:FVI}, Variance-weighted FVI reweights each sample via an estimated conditional variance $\widehat{\sigma}^2$
{\small
\begin{align*}
\widehat{{w}}_{h}:=\underset{{w} \in \mathbb{R}^{d}}{\operatorname{argmin}} \;&\sum_{k=1}^{n}\frac{\left[\langle {\phi}(s_{h}^k,a_{h}^k), {w}\rangle-r_{ h}^k-\widehat{V}_{h+1}(s_{h+1}^{\prime k})\right]^{2}}{\widehat{\sigma}^2_h(s_h^k,a_h^k)} \\
&+  \lambda\|{w}\|_{2}^{2}
\end{align*}}where $\widehat{\sigma}^2$ approximates $\Var_{P_{s,a}}(r+V^*_{h+1})$ and can be computed by estimating the first and second order moments separately. The pessimism in this case is modified as:{\small\[
\Gamma_h \approx O\left(\sqrt{d} \cdot (\phi(\cdot, \cdot)^{\top} \widehat{\Lambda}_{h}^{-1} \phi(\cdot, \cdot) )^{1 / 2}\right)+\frac{ H^4\sqrt{d}}{n}
\]
}with $
\widehat{\Lambda}_h=\sum_{k=1}^n\phi(s_{h}^k,a_{h}^k)\phi(s_{h}^k,a_{h}^k)^\top/\widehat{\sigma}^2_h(s_{h}^k,a_{h}^k) +\lambda I_d$ being the reweighted Gram matrix. With the update $Q_h(\cdot,\cdot)\leftarrow \phi(\cdot,\cdot)^\top \widehat{w}_h-\Gamma_h(\cdot,\cdot)$, we have the following.

\begin{theorem}\label{thm:main-linear}
    For linear MDPs, under assumption~\ref{assume:lopl} and some mild conditions, with high probability, for all policy $\pi$ simultaneously, $v^*-v^{{\pi}^{\text{Vw-PFVI}}}$ is bounded by
	\[
    \begin{aligned}
&\widetilde{O}\big(\sqrt{d}\sum_{h=1}^H\mathbb{E}_{\pi}\bigg[\sqrt{\phi(\cdot, \cdot)^{\top} \Lambda_{h}^{-1} \phi(\cdot, \cdot)}\bigg]\big)+\frac{2 H^4\sqrt{d}}{n},
        \end{aligned}
	\]
where $\Lambda_h=\sum_{k=1}^K \frac{\phi(s_{h}^k,a_{h}^k)\cdot \phi(s_{h}^k,a_{h}^k)^\top}{\sigma^2_{\widehat{V}_{h+1}(s_{h}^k,a_{h}^k)}}+\lambda I_d$. Moreover, $v^*-v^{{\pi}^{\text{Vw-PFVI}}}$ is also bounded by
	\begin{equation}\label{eqn:optimal_eqn}
\begin{aligned}
&\widetilde{O}\bigg(\sqrt{d}\cdot\sum_{h=1}^H\mathbb{E}_{\pi^*}\bigg[\sqrt{\phi(\cdot, \cdot)^{\top} \Lambda_{h}^{*-1} \phi(\cdot, \cdot)}\bigg]\bigg)+\frac{2 H^4\sqrt{d}}{n},
    \end{aligned}
	\end{equation}
where $\Lambda^*_h=\sum_{k=1}^K \frac{\phi(s_{h}^k,a_{h}^k)\cdot \phi(s_{h}^k,a_{h}^k)^\top}{\sigma^2_{{V}^*_{h+1}(s_{h}^k,a_{h}^k)}}+\lambda I_d$ and $\widetilde{O}$ hides universal constants and the Polylog terms.
\end{theorem}

Theorem~\ref{thm:main-linear} extends the instance-dependent characterization for offline RL in \ref{sec:tabular_instance_optimality} to the linear case. Compared to FVI \eqref{eqn:PFVI}, the main term in Theorem~\ref{thm:main-linear} replaces the explicit dependence on $H$ with a more adaptive/instance-dependent characterization. For instance, if we ignore the technical treatment by taking $\lambda=0$ and $\sigma^*_h\approx \Var_P(V^*_{h+1})$, then for the {partially deterministic systems} (where there are $t$ stochastic $P_h$'s and $H-t$ deterministic $P_h$'s), the main term diminishes to $$\sqrt{d}\sum_{i=1}^t\mathbb{E}_{\pi^*}\big[\sqrt{\phi(\cdot, \cdot)^{\top} \Lambda_{h_i}^{*-1} \phi(\cdot, \cdot)}\big]$$ 
with $h_i\in\{h:\;s.t. \;P_{h}\;\text{is stochastic}\}$ and can be a much smaller quantity when $t\ll H$. Furthermore, for the {fully deterministic system}, \ref{thm:main-linear} automatically provides faster convergence rate $O(\frac{1}{n})$, given that the main term degenerates to $0$. 

\subsection{OPL with Parametric Function Approximation}\label{sec:OPL_parametric} 

The parametric function class $\mathcal{F}:=\{f(\theta,\phi(\cdot,\cdot)):\mathcal{S}\times\mathcal{A}\rightarrow\mathbb{R},\theta\in\Theta\}$ provides the flexibility of selecting model $f$, making it possible for handling a variety of tasks. For instance, when $f$ is instantiated to be neural networks, $\theta$ corresponds to the weights of each network layers and $\phi(\cdot,\cdot)$ corresponds to the state-action representations (which is induced by the network architecture). When facing with easier tasks, we can deploy simpler model $f$ such as polynomials or even linear function $f(\theta,\phi)=\langle \theta,\phi \rangle $. 

Similar to FQE for the parametric function approximation, FQI perform the update with pessimism ($\phi_{h,k}=\phi(s^k_h,a^k_h)$) 
{\small
\begin{align*}
\widehat{\theta}_h&\leftarrow \argmin_{\theta\in\Theta}\sum_{k=1}^{n}\frac{\left[f\left(\theta, \phi_{h,k}\right)-r_{h,k}-\widehat{V}_{h+1}(s_{h+1}^k)\right]^{2}}{\widehat{\sigma}^2_h(s_h^k,a_h^k)}+\lambda \norm{\theta}_2^2\\
&\Gamma_{h}(\cdot, \cdot) \leftarrow \widetilde{O}\left(d\sqrt{\nabla_{\theta} f(\widehat{\theta}_h,\phi(\cdot,\cdot))^{\top} {\Lambda}_{h}^{-1} \nabla_{\theta} f(\widehat{\theta}_h,\phi(\cdot,\cdot))}+\frac{1}{K}\right),
\end{align*}} where $\widehat{\sigma}^2_h$ approximates $\sigma^2_h(s,a):=\Var_{P(\cdot|s,a)} (r+V^*_{h+1})$, and the reweighted Gram matrix has the form 
$${\Lambda}_{h} \leftarrow \sum_{k=1}^{n} \nabla f(\widehat{\theta}_h,\phi_{h,k}) \nabla f(\widehat{\theta}_h,\phi_{h,k})^{\top}/\widehat{\sigma}^2(s^k_h,a^k_h)+\lambda \cdot I.$$ 
Compared to Linear function approximation, the feature representation $\phi$ is replaced with $\nabla f(\widehat{\theta},\phi)$, and regression objective admits no closed-form solution.
This design generalizes the results in linear function approximation as follows:

\begin{theorem}[\cite{yin2022offline}]\label{thm:VAFQL}
	Suppose Assumption~\ref{assum:R+BC},\ref{assum:cover} and other mild conditions, with probability $1-\delta$, for all policy $\pi$ simultaneously, it holds ($\phi_h=\phi(s_h,a_h)$)
	\begin{align*}
	&v^\pi-v^{\widehat{\pi}}\lesssim 
 d\sum_{h=1}^H \E_\pi\left[\norm{\nabla_\theta f(\widehat{\theta}_h,\phi_h)}_{\Lambda_h^{-1}}  \right]
 +\frac{1}{n}.
	\end{align*}
 In particular, it has  
	\[
	v^{*}-v^{\widehat{\pi}}\lesssim  d\sum_{h=1}^H \E_{\pi^*}\left[\norm{\nabla^\top_\theta f({\theta}^*_h,\phi_h)}_{\Lambda^{*-1}_h}\right]+\frac{1}{n}.
	\]
	Here $\Lambda^{*}_h=\sum_{k=1}^K \frac{\nabla_\theta f({\theta}^*_h,\phi_{h,k})\nabla^\top_\theta f({\theta}^*_h,\phi_{h,k})}{\sigma^*_h(s^k_h,a^k_h)^2}+\lambda I_d$ and the $\sigma^*_h(\cdot,\cdot)^2:=\max\{1,\Var_{P_h}V^*_{h+1}(\cdot,\cdot)\}$.
	
\end{theorem}

From a technical perspective, the key tool for finite-sample analysis in function approximation is the \emph{Self-Normalized Concentration for Vector-Valued Martingales} \cite{abbasi2011improved}, originally developed for analyzing stochastic linear bandits. This tool provides a Hoeffding-style concentration bound that does not rely on variance or second-order information. Recently, Zhou et al. \cite{zhou2021nearly} extended this by proving a Bernstein version of Self-Normalized Concentration for linear mixture MDPs. This approach is well-suited for analyzing variance reweighting mechanisms in offline RL, applicable to both linear MDPs \cite{xiong2022nearly,yin2022near} and parametric models \cite{yin2022offline} as discussed above.

\subsection{Pessimism in the wild}

Beyond the theoretical focus, the aim of pessimism is to explicitly account for uncertainty in state-action value estimation and ``penalize'' actions in areas of high uncertainty. This approach generally involves modifying the value function (or policy optimization procedure) to discourage actions that have high uncertainty. There are several ways to implement pessimism:

\emph{Lower Confidence Bound (LCB)}. Instead of using the point estimate of the value function, the agent computes a lower bound based on the confidence interval around the estimate. If the agent is uncertain about the true value 
$Q(s,a)$, it will use a pessimistic estimate such as:
\[
Q_{\text{LCB}}(s, a)=\hat{Q}(s, a)-\lambda \cdot U(s, a)
\]
with $\lambda$ controlling the level of pessimism. This encourages the agent to favor actions with more reliable estimates, avoiding overestimated, risky actions, and is celebrated by theoretical research \cite{rashidinejad2021bridging,xie2021policy,wang2022gap,nguyen2023instance,di2023pessimistic} and other research mentioned in the previous sections. 

\emph{Penalty on Critic.} Another approach is to add a divergence penalty term to the critic objective to make conservative value estimates \cite{nachum2019algaedice}. For instance, \cite{kostrikov2021offline} uses Fisher divergence with respect to the Boltzmann policy and behavior policy, and \emph{conservative Q-learning} \cite{kumar2020conservative,lyu2022mildly} uses Kullback–Leibler divergence for the Boltzmann policy and the behavior policy.

\emph{Policy Regularization.} Regularization is often used to prevent the learned policy from deviating too much from the behavior policy. This can be viewed as a pessimistic strategy because the learned policy is constrained to stay close to what has been observed, reducing the risk of taking untested actions. For instance, BRAC \cite{wu2019behavior} adds a regularization term $\mathbb{E}_{s \sim \mathcal{D}}\left[D_{\mathrm{KL}}\left(\pi_\theta(\cdot \mid s) \| \pi_b(\cdot \mid s)\right)\right]$ to prevent the learned policy from deviating too much from the behavior policy, thus ensuring pessimistic behavior in uncertain areas.

\textbf{Optimism vs Pessimism?} 
While, under the offline setting, the pessimistic algorithm is consistent with rational decision-making using preferences that satisfy uncertainty aversion \cite{gilboa1989maxmin}, it remains intriguing whether pessimism is uniformly better than optimism in the instance-dependent scenarios. For multi-armed bandit problems, \cite{xiao2021optimality} demonstrated that greedy, optimistic, and pessimistic approaches are all (globally) minimax optimal for offline optimization, with each potentially outperforming the others in specific instances. For example, in case 1, where batch data frequently includes good arms, pessimism performs better. Conversely, in case 2, if the behavior policy pulls good arms infrequently, optimism may be advantageous due to the higher uncertainty associated with good arms. This suggests that existing instance-dependent offline RL studies primarily address case 1 (echo Assumption~\ref{assum:single_concen}), leaving open the question of whether section~\ref{sec:AF-offlineRL} could be further enhanced by incorporating an optimistic perspective, as in case 2.

\section{Low-Adaptive Exploration in RL}\label{sec:lowadaptive}
So far, we have focused on offline RL which aims at doing the best one can with the given data in learning a new policy. The resulting algorithm is based on the ``pessimism'' principle that discourages exploration. 
\begin{figure}[tb]
    \centering
    \includegraphics[width=0.6\linewidth]{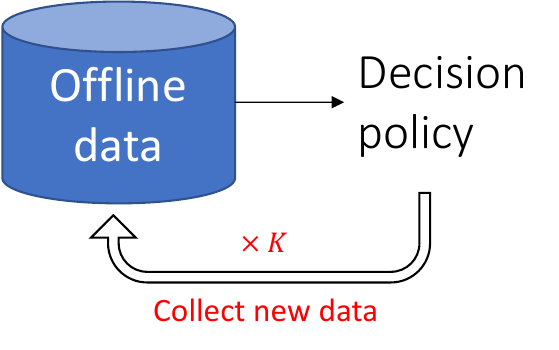}
    \caption{Illustration of the problem of low-adaptive RL.}
    \label{fig:illus_lowadaptive}
\end{figure}

If an offline RL agent ends up finding a near-optimal policy, that is because we are lucky to have observed data that covered all states/actions that that optimal policy has taken.  Alternatively, we can change the goal-post (in ``assumption-free'' offline RL) by declaring that we will only learn that part that is ``observable'' based on the offline data.  Statistical lower bounds indicate that both these results cannot be substantially improved.

This conclusion is quite pessimistic indeed in that it does not take into account the common real-life scenario that the learned policy may get deployed and that a new batch of data will eventually be collected, nor how to make use of the new data.

This is a much weaker claim than the online RL, which involves algorithmically ensuring that the exploration policies to have good coverage, hence allowing the learner to identify the optimal policy.

One way to think about this is that offline RL is a problem with no adaptivity allowed, while online RL allows adaptively choosing a new policy after every trajectory. This motivates us to consider the problem \emph{in between} by asking:
\begin{center}
    \textsf{Can we learn as well as the best online RL agent while using only a few batches? } 
\end{center} 
One can also think about the problem as a sequence of offline RL problem, but the learner can decide on the exploration policy $\mu$ to run for the next batch. All three examples that we considered as motivation of offline RL in the introduction may actually allow some limited exploration. Minimizing the number of batches help to alleviate all of the following issues.
	\begin{itemize}
		\item \textbf{Deployment Costs:} Updating policies in distributed systems, such as autonomous vehicles or network routers, can be computationally expensive.
		\item \textbf{Testing and Approval Overheads:} Policies in sensitive domains (e.g., healthcare) require extensive testing, ethical approvals, and regulatory compliance.
		\item \textbf{Concurrency Challenges:} Running experiments in parallel to identify optimal policies is limited by physical and logistical constraints.
	\end{itemize}
	Low-adaptive RL addresses these challenges by limiting the number of policy changes ($K$) during the learning process, where $K \ll T$ (the total number of rounds). 
This problem is well-studied in multi-armed bandits and linear bandits which shows that no-regret learning with $\tilde{O}(\sqrt{T})$ regret can be achieved with only $O(\log\log T)$ batches of exploration \citep{cesa2013online,perchet2016batched,gao2019batched}, but the same problem on RL is only getting started recently \citep{matsushima2021deployment,huang2021towards,qiao2022sample,qiao2023near}.
    

There are two closely related settings.
\begin{description}
    \item[RL with low switching cost] The learner must limit the number of times the deployed policy changes to $K$ (for historical reason the policy is often confined to deterministic policies). 
    \item[RL with low batch complexity] The learner must schedule $K$ batches of exploration (i.e., experiments) ahead of time and only look at the collected data in the completed batch and decide on the (sequence of) policies to use for the next batch at pre-determined checkpoints.
\end{description}
Both settings could make sense in practice with the second setting being qualitatively stronger\footnote{It is stronger when we allow randomized policies, and not compatible if we restrict to deterministic policies.} as monitoring certain statistics might be much cheaper than deploying new policies.

\textbf{Regret and sample complexity in online RL.} Considering low-switching or low batch complexity in isolation does not make sense, the algorithm must also be able to find a near-optimal policy. To quantify the performance of an RL algorithm it is typical that we consider (cumulative regret) of the sequence of policies being played $\pi^{(t)}$
$$
\mathrm{Regret} := \sum_{t=1}^T v^{\pi^*} -  v^{\pi^{(t)}}
$$
or the number of samples $T$ as a function of $\epsilon>0$ such that we can identify $\hat{\pi}$ that satisfies 
$$
v^{\pi^*} -  v^{\hat{\pi}}\leq \epsilon.
$$

 In the remainder of this section, we survey the existing work on reinforcement learning for both the tabular case and under function approximation. 

\subsection{Learning Tabular RL in $O(\log \log T)$ batches}\label{sec:batched_tabular}
Let us first state the known information-theoretic lower bounds in this problem.
\begin{theorem}
    Consider the tabular RL problems (defined in Section~\ref{sec:setup}). Assume $S<A^{H/2}$ \footnote{This is without loss of generality because otherwise uniform exploration and IS-based OPE with curse-of-horizon suffices to solve the problem with one batch.}.  
    \begin{enumerate}[leftmargin=*] 
        \item  Any algorithms with a regret of $\tilde{O}(\sqrt{T\mathrm{poly}(H,S,A)})$ must incur a switching cost of $\Omega(HSA\log\log T)$ and use $\Omega(H/\log T + \log\log T)$ batches of exploration. 
        \item Moreover, any algorithms with a regret of $o(T)$ must incur a switching cost of $\Omega(HSA)$ and use $\Omega(H/\log T)$ batches of exploration.
    \end{enumerate} 
\end{theorem}
The lower bounds of for the batch complexity is due to Theorem B.3 of \cite{huang2021towards} and Corollary 3 of \citep{gao2019batched}. The lower bounds of the switching costs are due to \citep[Theorem 4.2 and 4.3]{qiao2022sample}.

Now let us inspect the algorithmic techniques in this space. First, a doubling schedule of exploration due to the UCB2 algorithm \citep{auer2002finite} can be combined with optimistic exploration $Q$-learning to obtain a near-optimal regret while using only $\log(T)$ switching cost\citep{bai2019provably}, but since it requires monitoring the exploration to decide when to change the policy, its batch complexity remains $T$. 

Can this algorithm be improved? \citet{qiao2022sample} designed a policy elimination-based method called \emph{Adaptive Policy Elimination by Value Estimation} (APEVE) that achieves the following guarantees:
	\begin{itemize}
		\item \textbf{Near-Optimal Regret} 
			$\tilde{O}(\sqrt{H^4 S^2 A T})$ which is optimal up to a factor of $HS$.
		\item \textbf{Switching Costs} of 
        $O(H S A \log \log T)$
		matching the information-theoretic lower bound.
        \item  \textbf{Batch complexity} of $O(H\log\log T)$, which matches the information-theoretic lower bound in $T$ \footnote{A minor variation of APEVE called APEVE+ achieves \textbf{Batch complexity} of $O(H + \log\log T)$ \citep{qiao2022sample}.}
	\end{itemize}
These results highlight that low-adaptive RL can achieve comparable performance to traditional online RL while using only a small number of batches.

APEVE is a \textit{policy elimination} method, which iteratively narrows down the set of candidate policies by eliminating those deemed suboptimal. The method combines:
\begin{enumerate}[leftmargin=*]
    \item \textbf{Crude Layer-Wise Exploration:} A coarse-grained exploration scheme that explores each $h,s,a$ layer by layer. This provides a crude approximation of the useful part of the MDP's transition kernel. 
    \item \textbf{Fine Stagewise Exploration:} Use the crude transition kernel estimate to plan and identify $HSA$ policies that each visits a particular triplets $h,s,a$ most frequently among all policies in the remaining set of policies, then execute these policies to collect more data. 
    \item \textbf{Confidence-Bound Based Elimination:} Use the dataset with good coverage to conduct OPE on all policies that remains to be contenders, then eliminate those policies with their upper confidence bound lower than the highest lower confidence bound.
\end{enumerate}
Let the total number of stages be $K$, and the $k^{th}$ stage have length $T^{(k)} = K^{1-1/k}$, one can work out that the smallest $K$ such that $\sum_{k=1}^K T^{(k)} > T $ is $K=O(\log\log T)$. The total number of stages is only $O(\log\log T)$ and in each stage, it requires deterministically changing policies for $HSA$ times per stage.
 
\textbf{Reward-free exploration with $O(H)$-batches.} \citet{qiao2022sample} also presented a reward-free exploration method (LARFE) with a sample complexity of $O(H^5S^2A / \epsilon^2)$ for identifying any policies while using only $2H$ rounds of adaptivity. LARFE does not need to perform the $\log\log T$ stages of exploration since it does not care about regret, so the crude-layerwise exploration can reach a reasonable approximation and the $HSA$ exploration policies can be identified at one shot for driving the error down.

These results demonstrate that there are algorithms that can achieve nearly the same regret or sample complexity as the best online algorithm even if we only give a very small room for adaptively updating the policies. The result is further improved in \cite{zhang2022near}, who improved the regret bound to the optimal $\tilde{O}(\sqrt{H^3SA T})$ while retaining the same batch complexity.

\subsection{Linear function approximation and Reward-Free Exploration in $O(H)$ batches}\label{sec:batched_linear}
The natural next question is whether APEVE-like algorithms can be derived for RL under linear function approximation. The lower bounds are in place, 
\begin{theorem}[Theorem~7.2 and 7.2 of \citep{qiao2023near}]
    Under linear MDPs setting, any algorithm that achieves $\tilde{O}(\sqrt{T\mathrm{poly}(d,H)})$ regret must incur a switching cost of $\Omega(dH\log\log T)$ and a batch complexity of $\Omega(H/\log d + \log\log T)$
\end{theorem}
Unfortunately, there are technical challenges and the best low-adaptive learner of linear MDPs for regret minimization still requires $O(\log T)$ batches from the doubling trick \citep{wang2021provably,gao2021provably} using the doubling trick from \citet{abbasi2011improved}.

On the other hand, in the reward-free exploration setting, a policy elimination approach \citep{qiao2023near} with merely $H$ batches of exploration while achieving a sample-complexity bound of $O(d^2H^5/\epsilon^2)$.   This improves over a related result \citep{huang2021towards} that obtains $O(d^3H^5/\epsilon^2\nu_{\min}^2)$ where $\nu_{\min}$ is an (arbitrarily small) problem-specific reachability parameter. The algorithm of \citep{qiao2023near} is also more satisfying as it does not need to know  $\nu_{\min}$ and the result does not deteriorate as $\nu_{\min}$ gets smaller.

The key algorithmic ideas are closely related to the reward-free exploration algorithm (LARFE) for the tabular case that uses layer-wise exploration (which gives rise to $H$ batches of exploration), with a carefully chosen batch of exploration policy for the next layer after knowing the MDP parameters for the current layer. 

The main difference from the tabular case is that instead of estimating the transition kernels as discrete probability distributions, we now solve linear regression problems.  Instead of identifying the policies that maximizes the visitation measure to every $(h,s,a)$, we identify a set of policies $\Pi_{h,\epsilon}$ that maximizes the visitation to every direction of features $\phi(h,s,a)$ that is relevant to learning while still keeping the set relatively small. Then the batched exploration policy $\pi$ that can be obtained using a variant of G-optimal experiment design that minimizes the maximum ``misalignment'' of the covariance matrix, namely, $\max_{\pi' \in \Pi_{h,\epsilon}}\E_{\pi'}[ \phi(s,a)^T \Sigma_\pi \phi(s,a)]$. This is still infeasible because $\pi'$ is not executed, but we can estimate the $\E[\cdot]$ uniformly for every $\pi'\in\Pi_{h,\epsilon}$ and showed that the approximate G-optimal design still works.

\subsection{Beyond Linear MDPs}\label{sec:batched_general}
Low-adaptive RL beyond linear function approximation is more open-ended. 
Most existing work settles with $O(\log T)$-style switching cost bounds that generalizes the ``doubling trick'' to more abstract settings such as linear Bellman-complete MDPs with low inherent Bellman error \citep{qiao2024logarithmic} or low Bellman Eluder-dimension \citep{zhao2024nearly}.  There hasn't been any algorithm that achieves no regret learning with either $O(\log\log T)$ switching cost or $O(\log\log T)$ batches of exploration. This is a major open problem in this space. The best-policy identification problem is likely to be easier. We believe reward-free exploration in the low-adaptive case is tractable by combining techniques from \citep{qiao2023near} and \citep{yin2022offline}.

\section{Conclusion and Open problems} 
In this paper, we have surveyed recent advances in the statistical theory of offline reinforcement learning as well as the related problem of low-adaptive exploration. Both problems are well-motivated by the emerging applications of reinforcement learning for real-life sequential decision-making problems. We covered results that characterize the optimal statistical complexity of each problem family as well as algorithms that are not only minimax optimal but also adaptive to individual problem instances across a hierarchy of coverage assumptions and structural conditions. We described not only the technical results but also theoretical insights on how these algorithms work and where the technical challenges are. 

We conclude the paper by highlighting a few open directions of research in this rich problem space. 
\begin{itemize}
    \item \textbf{Agnostic Offline RL with function approximation.} Most provable offline RL algorithms in the function approximation settings require strong assumptions on the realizability and self-consistency (i.e., Bellman completeness) of the given function class.  In practice, it is observed that even when linear function approximation is a poor approximation, the resulting policy that one can learn with it under a realistic exploration budget is still very impressive. At the moment there is no appropriate theoretical framework that satisfactorily quantifies this behavior. It will be nice to understand how much we can push the theoretical limit towards achieving similar levels of agnostic learning for offline (and online) RL comparable to supervised learning.
    \item \textbf{$O(\log\log T)$-adaptive RL with function approximation}  As we described in Section~\ref{sec:batched_linear} it remains open even under linear MDP how to achieve the optimal $O(\log\log T)$ batch complexity or switching cost while achieving a $\tilde{O}(\sqrt{T})$ regret. This is a concrete open problem that we hope to see resolved in the next few years.
    \item \textbf{Efficient computation} The paper focuses on the information-theoretical aspects of the problems and does not distinguish whether the OPE estimators,  offline RL algorithms or the low-adaptive online learners are efficiently computable. For offline RL, anything beyond linear MDPs are computationally intractable. For low-adaptive RL, the algorithms are inefficient even for the tabular case (except in some cases when there are linear-program reformulations of the experiment-design).
    \item \textbf{Theory-inspired algorithms in offline Deep RL} Despite the widely-recognized importance of offline RL problems, the theory and practice remain pretty disjoint. The principle of ``pessimism'' is independently discovered but the theoretically approaches for implementing ``pessimism'' and deep RL heuristics for implementing ``pessimism'' are very different \cite{levine2020offline,kumar2020conservative,li2023offline,asadi2024learning}. The Deep RL heuristics are often overly optimized to the specific test cases in popular benchmarks and do not work well in new problems. This was demonstrated in the context of RL for computer networking \citep{haider2024networkgym} and that an simple alternative algorithm inspired by the pessimistic bonus of \cite{yin2022offline} turns out to work significantly better than state-of-the-art deep RL counterparts. We believe it is a productive avenue of research to bring some of the theoretical ideas from offline and low-adaptive RL to practice in different problem domains.
\end{itemize}

\begin{appendix}

\section{Examples of ``Curse of Horizon'' for Importance Sampling estimators}\label{app:CoH_examples}
In this Appendix, we provide two concrete examples where the IS estimators suffer from the ``Curse of Horizon''.

\textbf{Example 1.}\cite{liu2018breaking}
Consider a ``ring MDP'' with $n$ (an odd number) states $\mathcal{S} = \{0,1,\cdots,n-1\}$, arranged on a circle (see the figure on the right). There are two actions for all states, ``L'' and ``R''. The L action moves the agent from the current state counterclockwise to the next state, and the R action does the opposite direction. This can be equivalently written as: 
\[
\begin{aligned}
P(s^{\prime} \mid s, \mathrm{~L}) & =\mathbb{I}(s^{\prime}=s-1 \bmod n) \\
P\left(s^{\prime} \mid s, \mathrm{R}\right) & =\mathbb{I}\left(s^{\prime}=s+1 \bmod n\right).
\end{aligned}
\]
Let $\eta\in[0,1]$ and $\eta\neq 1/2$. We choose the behavior policy $\mu$ and target policy $\pi$ as follows: $
    \pi(\text{R}|s)=\mu(\text{L}|s)=1-\eta, \; \; \mu(\text{R}|s)=\pi(\text{L}|s)=\eta$.

\begin{proposition}\label{prop:break}
    Variance of cumulative ratio $\rho_{1:H}$ grows exponentially in $H$. Formally, $\text{Var}_{\mu}[\rho_{1:H}]=A_\eta^H-1$ with $A_\eta=\frac{\eta^3+(1-\eta)^3}{(1-\eta)\eta}>1$. Similarly, it further holds $\text{Var}_{\mu}[\widehat{v}^\pi_{\text{IS}}]=\Theta(A_\eta^H)$.
\end{proposition}

\begin{proof}
Denote $C=(1-\eta)/\eta$ and $\boldsymbol{\tau}$ to be the random trajectory, then $F(\boldsymbol{\tau})=\sum_{t=1}^H \mathbb{I}(a_t=\text{R})$ follows a Binomial distribution $Binomial(H,\eta)$. Furthermore, the relation holds that 
\[
\rho_{1:H}(\boldsymbol{\tau})=\prod_{t=1}^H\frac{\pi(a_t|s_t)}{\mu(a_t|s_t)}=\left(\frac{1-\eta}{\eta}\right)^{2F(\boldsymbol{\tau})-H}=C^{2F(\boldsymbol{\tau})-H}.
\]
Note $F(\boldsymbol{\tau})\sim Bin(H,\eta)$ implies $\E_{\boldsymbol{\tau}\sim \mu}[\rho_{1:H}(\boldsymbol{\tau})]=1$, and the second order moment
\begin{align*}
&\mathbb{E}_{\boldsymbol{\tau} \sim \mu}\left[\rho_{1:H}(\boldsymbol{\tau})^2\right]  =\mathbb{E}_{\boldsymbol{\tau} \sim p_{\pi_0}}\left[(C^{2 F(\boldsymbol{\tau})-H})^2\right]= \\
 &\Phi(4 \log C) \cdot C^{-2H} =\left[\left(1-\eta+\eta C^4\right) C^{-2}\right]^H =A_\eta^{H}.
\end{align*}
Here $\Phi$ is the moment generating generating function of Binomial distribution ($\forall \lambda\in\R$): 
\[
\Phi(\lambda):=\mathbb{E}_{\boldsymbol{\tau} \sim \mu}[\exp (\lambda F(\boldsymbol{\tau}))]=(1-\eta+\eta \exp (\lambda))^{H}
\]
Therefore, the variance is $A_\eta^{H}-1$ which is exponential in $H$. Besides, $\text{Var}_{\mu}[\widehat{v}^\pi_{\text{IS}}]=\Theta(A_\eta^H)$ can be proved similarly. 
\end{proof}

\textbf{Example 2. \cite{xie2019towards}} For the second example, we can consider an MDP with i.i.d. state transition and constant sparse reward $1$ shown at the last step. The IS estimator becomes $\widehat{v}_{\mathrm{IS}}^\pi=\frac{1}{n} \sum_{i=1}^n[\prod_{t=1}^H \frac{\pi\left(a_t^{(i)} \mid s_t^{(i)}\right)}{\mu\left(a_t^{(i)} \mid s_t^{(i)}\right)}]$. Suppose $\log \frac{\pi_t}{\mu_t}$ is bounded (or equivalently $\frac{\pi_t}{\mu_t}$ is bounded from both sides) with $E_{\log }=\mathbb{E}[\log \frac{\pi_t}{\mu_t}]$ and $ V_{\log }=\operatorname{Var}[\log \frac{\pi_t}{\mu_t}]$. By Central limit theorem, random variable $\sum_{t=1}^H\frac{\pi_t}{\mu_t}\sim \mathcal{N}(HE_{\log },HE_{\log })$ asymptotically, and this is the same as $\prod_{t=1}^H \frac{\pi_t}{\mu_t}\sim\text{LogNormal}(HE_{\log },HV_{\log })$. This comes from the state transitions are i.i.d. The variance of $\prod_{t=1}^H \frac{\pi_t}{\mu_t}$ is again exponential in horizon $\Theta(\exp(HV_{\text{log}}))$.

Both examples have finite number of states and actions, which demonstrates that IS-based estimators suffer from exponential variance even for the simplest tabular RL.

As we discussed, there are other estimators that do not suffer from the curse of horizon for these problems, but they all require the value functions to be easily estimable (with a small state space being a special case).

\end{appendix}

\bibliographystyle{plainnat}
\bibliography{sample.bib}

\end{document}